\def\year{2021}\relax
\documentclass[letterpaper]{article} % DO NOT CHANGE THIS
\usepackage{aaai21}  % DO NOT CHANGE THIS
\usepackage{times}  % DO NOT CHANGE THIS
\usepackage{helvet} % DO NOT CHANGE THIS
\usepackage{courier}  % DO NOT CHANGE THIS
\usepackage[hyphens]{url}  % DO NOT CHANGE THIS
\usepackage{graphicx} % DO NOT CHANGE THIS
\usepackage{natbib}
\usepackage{caption}
\usepackage{amsfonts,nicefrac,microtype}      % microtypography
\usepackage{multirow,color}
\usepackage{subcaption}
\usepackage{algorithm,algorithmic}
\usepackage{amsmath}\allowdisplaybreaks
\usepackage{amsfonts,bm}
\usepackage{amssymb}

\def\eqref#1{equation~(\ref{#1})}
\def\1{\bf{1}}

\newcommand{\Norm}[1]{\left\| #1 \right\|}
\newcommand{\norm}[1]{\left\| #1 \right\|_2}

\def\eps{{\varepsilon}}

\def\vzero{{\bf{0}}}

\def\va{{\bf{a}}}

\def\vu{{\bf{u}}}
\def\vv{{\bf{v}}}

\def\vx{{\bf{x}}}
\def\vy{{\bf{y}}}

\def\fN{{\mathcal{N}}}
\def\fO{{\mathcal{O}}}

\def\fQ{{\mathcal{Q}}}

\def\fY{{\mathcal{Y}}}
\def\fZ{{\mathcal{Z}}}

\def\BP{{\mathbb{P}}}

\def\BR{{\mathbb{R}}}

\newcommand{\E}{\mathbb{E}}

\newcommand{\R}{\mathbb{R}}

\DeclareMathOperator*{\argmax}{arg\,max}

\DeclareMathOperator{\diag}{diag}

\usepackage{amsthm}
\theoremstyle{plain}
\newtheorem{thm}{Theorem}%[section]

\newtheorem{lem}{Lemma}

\makeatletter
\def\Ddots{\mathinner{\mkern1mu\raise\p@
\vbox{\kern7\p@\hbox{.}}\mkern2mu
\raise4\p@\hbox{.}\mkern2mu\raise7\p@\hbox{.}\mkern1mu}}
\makeatother

\makeatletter
\newcommand*{\rom}[1]{\expandafter\@slowromancap\romannumeral #1@}
\makeatother

\def\SPM {{\rm{SubspacePowerMethod}}}
\def\nnz {{\rm nnz}}
\def\tr {{\rm tr}}
\def\ut {{(t)}}
\def\ui {{(i)}}
\def\uj {{(j)}}
\def\utm {{(t-1)}}
\def\A {{\bf A}}
\def\B {{\bf B}}
\def\C {{\bf C}}
\def\D {{\bf D}}
\def\E {{\bf E}}

\def\G {{\bf G}}

\def\I {{\bf I}}
\def\K {{\bf K}}
\def\M {{\bf M}}
\def\N {{\bf N}}

\def\P {{\bf P}}
\def\Q {{\bf Q}}
\def\R {{\bf R}}

\def\U {{\bf U}}
\def\V {{\bf V}}

\def\X {{\bf X}}
\def\Y {{\bf Y}}
\def\Z {{\bf Z}}
\def\hA{{\widehat \A}}
\def\hB{{\widehat \B}}
\def\hP{{\widehat \P}}
\def\hQ{{\widehat \Q}}
\def\tX{{\widetilde \X}}
\def\tY{{\widetilde \Y}}
\def\hX{{\widehat \X}}
\def\hY{{\widehat \Y}}
\def\QR {{\rm QR}}
\def\SVD {{\rm SVD}}

\def\tU{\widetilde{\U}}
\def\tV{\widetilde{\V}}
\def\tSigma{\widetilde{\mSigma}}
\def\hSigma{\widehat{\mSigma}}

\def\mSigma{{\bf{\Sigma}}}
\def\mOmega{{\bf{\Omega}}}

\title {Revisiting Co-Occurring Directions: Sharper Analysis and Efficient Algorithm
for Sparse Matrices}
\author {
    Luo Luo\textsuperscript{\rm 1},
    Cheng Chen\textsuperscript{\rm 2}\thanks{Corresponding Author},
    Guangzeng Xie\textsuperscript{\rm 3}, 
    Haishan Ye\textsuperscript{\rm 4} \\
}
\affiliations {
    \textsuperscript{\rm 1} 
    Department of Mathematics, The Hong Kong University of Science and Technology \\
    \textsuperscript{\rm 2} 
    Department of Computer Science and Engineering, Shanghai Jiao Tong University \\
    \textsuperscript{\rm 3} 
    Academy for Advanced Interdisciplinary Studies, Peking University \\
    \textsuperscript{\rm 4} 
    Shenzhen Research Institute of Big Data, The Chinese University of Hong Kong, Shenzhen  \\
    luoluo@ust.hk, jack\_chen1990@sjtu.edu.cn,
    smsxgz@pku.edu.cn, hsye\_cs@outlook.com \\
}

\begin{document}

\maketitle

\begin{abstract}
We study the streaming model for approximate matrix multiplication (AMM). We are interested in the scenario that the algorithm can only take one pass over the data with limited memory. The state-of-the-art deterministic sketching algorithm for streaming AMM is the co-occurring directions (COD), which has much smaller approximation errors than randomized algorithms and outperforms other deterministic sketching methods empirically. In this paper, we provide a tighter error bound for COD whose leading term considers the potential approximate low-rank structure and the correlation of input matrices.  We prove COD is space optimal with respect to our improved error bound. We also propose a variant of COD for sparse matrices with theoretical guarantees. The experiments on real-world sparse datasets show that the proposed algorithm is more efficient than baseline methods.
\end{abstract}

\section{Introduction}

A large scale machine learning system usually receives data sequentially and it is often impossible to exactly store the entire data set. Thus, the approximate matrix multiplication (AMM) in the streaming fashion is an important and fundamental task for scientific computation and big data analysis. For example, the product of matrices from multi-modal datasets captures the correlation between different modalities. In addition, many classical algorithms including canonical correlation analysis~\cite{hotelling1992relations},
generalized eigenvector decomposition~\cite{gene1996matrix},
partial least squares~\cite{wegelin2000survey}, spectral co-clustering~\cite{dhillon2001co} require to perform approximate matrix multiplication when the data set is very large.
On the other hand, data matrices from real-world are usually low-rank and sparse, which motivated us to design efficient and effective sparse algorithms. 

This paper considers streaming AMM problem as follows. Give two large matrices $\X\in\BR^{n\times d_x}$ and $\Y\in\BR^{n\times d_y}$, we are interested in finding a low-rank estimator $\A^\top\B$ to approximate $\X^\top\Y$, where $\A\in\BR^{m\times d_x}$, $\B\in\BR^{m\times d_y}$ and $m$ is much smaller than $n$, $d_x$ and $d_y$. 
We focus on the row update model, that is, the algorithm receives rows of $\X$ and $\Y$ sequentially and it only takes one pass over input matrices with limited memory. 
The key challenge for this problem is to reduce the space/time complexity while maintaining the approximation error.

Inspired by the idea of finding frequent items~\cite{misra1982finding}, \citet{liberty2013simple} proposed frequent directions algorithm (FD), which considers the symmetric case of AMM such that $\X=\Y$ (a.k.a., the covariance sketching). 
FD achieves optimal tradeoffs between space cost and approximation error~\cite{woodruff2014low,ghashami2016frequent,ghashami2014relative}. 
Moreover, we can combine FD with subspace power iteration~\cite{woodruff14sketching,musco2015randomized} to design an algorithm which is efficient for sparse matrix multiplication~\cite{ghashami2016efficient}, called sparse frequent directions (SFD). 
Recently, \citet{huang2019near} integrated random sampling~\cite{drineas2006fast} into FD to reduce its time complexity. \citet{luo2019robust} introduced a regularization term for FD, which makes the estimator is more friendly to inverse operation. FD technique can also be used to accelerate many popular machine learning models, such as convex online optimization~\cite{luo2016efficient,luo2019robust}, factorization machine~\cite{luo2018sketched}, linear contextual bandits~\cite{kuzborskij2019efficient,chen2020bandit} and ridge regression~\cite{shi2020deterministic,dickens2020ridge}.

\citet{mroueh2016co} proposed a variant of FD called co-occurring directions (COD) for streaming AMM. COD shrinks the singular values of input matrices $\X$ and $\Y$ simultaneously at each iteration. It is shown that COD has significantly better performance than other sketching algorithms~\cite{ye2016frequent,drineas2006fast,clarkson2017low,sarlos2006improved} on AMM problem empirically. 
However, the existing spectral error bound of COD can not completely explain its high performance. It depends on the Frobenius norm of $\X$ and $\Y$, which ignores the potential low-rank structure of the data matrix. Specifically, in the case of $\X=\Y$, the procedure of COD degrades to FD, but its error bound is worse than that of FD. 
Another deterministic sketching method for AMM, which we call FD-AMM~\cite{ye2016frequent}, directly adopts FD to sketch the concatenated matrix $\Z=[\X,\Y]$. The output of the algorithm is an approximation of $\Z^\top\Z$, whose sub-matrix corresponds to an estimator of $\X^\top\Y$. %We refer this method as FD-AMM~\cite{ye2016frequent}.

In this paper, we provide a sharper analysis for co-occurring directions (COD). We give a new spectral norm error bound which considers the potential low-rank structure of the target matrix. Our bound could be much tighter than \citeauthor{mroueh2016co}'s~(\citeyear{mroueh2016co}) results when the spectrum of the exact matrix product is dominated by its top singular values. In addition, we prove that the space complexity of COD is optimal to attain our improved error bound. Furthermore, in the case of $\X=\Y$, our result matches the error bound of FD. 

We further propose sparse co-occurring directions (SCOD) and provide an error bound matches our improved analysis on standard COD while the running time of the algorithm mainly depends on the non-zero entries of input matrices. We conduct numerical experiments on cross-language datasets to show that SCOD has better performance than state-of-the-art algorithms empirically. Concurrent to our work, \citet{wan2020approximate} have also proposed a similar COD based algorithm to address streaming AMM with sparse inputs but their error bound does not consider the potential low-rankness.  

The rest of the paper is organized as follows. 
In Section~\ref{section:preliminaries}, define the notation used in this paper and introduce the
background of related algorithms for streaming AMM. 
In Section~\ref{section:error-bound}, we provide our new error bound for COD algorithm and show the corresponding space lower bound. 
In Section~\ref{section:SCOD}, we propose SCOD and give its theoretical guarantees.
In Section~\ref{section:experiments}, we conduct the numerical experiments to show the superiority of SCOD. We defer detailed proof of some lemmas and theorems into supplementary materials.
We conclude our work in Section~\ref{section:conclusion}.

\section{Notations and Preliminaries}\label{section:preliminaries}

In this section, we first introduce the notation will be used in this paper. Then we give the backgrounds of frequent directions and related algorithms for AMM.

\subsection{Notations}

We let $\I_p$ be the $p\times p$ identity matrix and ${\bf 0}_{p\times q}$ be the $p\times q$ matrix of all zeros. For an $p\times q$ matrix $\A=[A_{ij}]$, we denote $(\va^\ui)^\top$ be its $i$-th row, $\nnz(\A)$ be the number of non-zero entries of $\A$. The condensed singular value decomposition (SVD) of $\A$ is defined as $\U\mSigma\V^\top$ where $\U\in\BR^{m\times r}$ and $\V\in\BR^{m\times r}$ are column orthogonal, $\mSigma=\diag(\sigma_1(\A),\sigma_2(\A),\dots,\sigma_r(\A))$ with $\sigma_1(\A)\geq\sigma_2(\A)\geq\dots\geq\sigma_r(\A)>0$ places the nonzero singular values on its diagonal entries and $r$ is the rank of $\A$. 
%The matrix pseudo-inverse of $\A$ is defined by $\A^\dagger=\V\bSigma^{-1}\U\in\BR^{d\times n}$.
We have $\sigma_i(\A)=0$ for any $i>r$.
Additionally, we let $\Norm{\A}_F =\sqrt{\sum_{i,j}A_{ij}^2}=\sqrt{\sum_{i=1}^r\sigma_i^2(\A)}$
be the Frobenius norm, $\norm{\A}=\sigma_1(\A)$ be the spectral norm, $\Norm{\A}_* =\sum_{i=1}^r\sigma_i(\A)$ be the nuclear norm and $\Norm{\A}_k =\sum_{i=1}^k\sigma_i(\A)$ be the Ky Fan $k$-norm. We also denote $\A_k$ as the best rank-$k$ approximation to $\A$ for any unitary invariant norms, that is, $\A_k=\sum_{i=1}^k \sigma_i(\A)\vu_i\vv_i^\top$, where $\vu_i$ and $\vv_i$ are the $i$-th column of $\U$ and $\V$ respectively.

\subsection{Frequent Directions}

Frequent directions~\cite{liberty2013simple,ghashami2016frequent} is a deterministic algorithm for covariance sketching. Given any matrix $\X\in\BR^{n\times d}$ and sketch size $m$ that is much smaller than $n$ and $d$, FD processes the rows of $\X$ one by one and produces a sketch matrix $\A\in\BR^{2m\times d}$ to approximate $\X^\top\X$ by $\A^\top\A$. We present the details of FD in Algorithm \ref{alg:FD}, which requires $\fO(md)$ space and $\fO(mnd)$ time complexity. The algorithm has the following theoretical guarantees. 
\begin{lem}[\citealt{ghashami2014relative,ghashami2016frequent}]\label{lem:FD}
    The output $\A$ of Algorithm \ref{alg:FD} satisfies
    \begin{align}\label{bound:FD}
        \norm{\X^\top\X-\A^\top\A} \leq \frac{1}{m-k}\left(\Norm{\X}_F^2-\Norm{\X_k}_F^2\right)
    \end{align}
    for any $k<m$.
\end{lem}
\citet{ghashami2016frequent} also prove FD is space optimal with respect to the guaranteed accuracy in Lemma~\ref{lem:FD}.
Note that the shrinking step in line~\ref{line:shrink} of the algorithm is necessary because the output could be extremely worse without this operation~\cite{desai2016improved,luo2019robust}.

\begin{algorithm}[ht] 
	\caption{Frequent Directions (FD)}\label{alg:FD}
	\begin{algorithmic}[1]
	    \STATE \textbf{Input:} $\X\in\BR^{n\times d}$ and sketch size $m$\\[0.1cm]
	    \STATE $\A\gets\vzero_{2m\times d}$ \\[0.1cm]
		\STATE \textbf{for} $t=1,2,\dots,n$ \\[0.1cm]
		\STATE\quad insert $(\vx^\ut)^\top$ into a zero valued row of $\A$ \\[0.1cm]
		\STATE\quad \textbf{if} $\A$ has no zero valued rows \textbf{then} \\[0.1cm]
        \STATE\label{line:fd1}\quad\quad $[\U, \mSigma, \V]\gets \SVD(\A)$ \\[0.1cm]
        \STATE\quad\quad $\delta \gets \sigma^2_m(\A)$ \\[0.1cm]
        \STATE\label{line:shrink}\quad\quad $\hSigma \gets \sqrt{\max\left(\mSigma^2-\delta\I_{2m},\vzero_{2m\times2m}\right)}$ \\[0.1cm]
        \STATE\label{line:fd2}\quad\quad $\A \gets \hSigma\V^\top$  \\[0.1cm]
        \STATE\quad \textbf{end if} \\[0.1cm]
		\STATE \textbf{end for} \\[0.1cm]
		\STATE \textbf{Output:} $\A$ \\[0.1cm]
	\end{algorithmic}
\end{algorithm}

\subsection{Sketching Algorithms for AMM}

It is natural to exploit the idea of FD to solve general AMM problem~\cite{ye2016frequent}. %We present this frequent directions based AMM (FD-AMM) method formally in Algorithm \ref{alg:FD-AMM}. FD-AMM 
We can concatenates the input matrix $\X\in\BR^{n\times d_x}$ and $\Y\in\BR^{n\times d_y}$ to construct a larger matrix $\Z=[\X,\Y]\in\BR^{n\times(d_x+d_y)}$, and then apply FD on $\Z$ to approximate $\Z^\top\Z$ by $\C^\top\C$, where $\C=[\A,\B]$, $\A\in\BR^{n\times d_x}$ and $\B\in\BR^{n\times d_y}$. 
The top right sub-matrix of the $\C^\top\C$, i.e.,the matrix $\A^\top\B$ is an approximation of $\X^\top\Y$. Intuitively, this algorithm wastes a large proportion of cost to approximate $\X^\top\X$ and $\Y^\top\Y$ (the other sub-matrices of $\Z^\top\Z$), which is unnecessary for the AMM task.
%We can use the result of Lemma~\ref{lem:FD} to construct the following error bound.\footnote{The original theoretical analysis of FD-AMM show that $\norm{\X^\top\Y-\A^\top\B} \leq \frac{1}{m}(\Norm{\X}_F^2+\Norm{\Y}_F^2)$, it is not difficult to improve this result to Lemma~\ref{lem:FD-AMM} by using \citeauthor{ghashami2014relative}'s trick.}
%\begin{lem}\label{lem:FD-AMM}
%The output $\A$ and $\B$ of Algorithm \ref{alg:FD-AMM} satisfies
%\begin{align}\label{bound:FD-AMM}
%    \norm{\X^\top\Y-\A^\top\B} \leq \frac{1}{m-k}\left(\Norm{\X}_F^2+\Norm{\Y}_F^2-\Norm{\Z_k}_F^2\right)
%\end{align}
%for any $k<m$, where $\Z$ is the concatenated matrix %$\Z=[\X,\Y]$.
%\end{lem}

\citet{mroueh2016co} proposed the
%another deterministic sketching algorithm called
co-occurring directions (COD) for AMM. We present its detailed procedure in Algorithm \ref{alg:COD}. Each iteration of COD constructs the column basis of $\A$ and $\B$ by QR factorization independently and executes the shrinkage step on the small interaction matrix $\R_x\R_y^\top$. 
We point out that both COD and FD-AMM requires $\fO(m(d_x+d_y))$ space and $\fO(mn(d_x+d_y))$ time complexity.
However, COD looks more reasonable than FD-AMM since all of its operations surrounds approximating $\X^\top\Y$. The numerical experiments~\cite{mroueh2016co} show that COD performs significantly better than FD-AMM~\cite{ye2016frequent} and other AMM algorithms~\cite{drineas2006fast,clarkson2017low,sarlos2006improved} when input matrices is dense. 
We can prove that COD holds the guaranteed accuracy as follows.
\begin{lem}[\citealt{mroueh2016co}]\label{lem:COD}
The output $\A$ and $\B$ of Algorithm \ref{alg:COD} satisfies
\begin{align}\label{bound:COD}
    \norm{\X^\top\Y-\A^\top\B} \leq \frac{\Norm{\X}_F\Norm{\Y}_F}{m}.
\end{align}
\end{lem}
Unfortunately, the result of Lemma~\ref{lem:COD} does not reveal the advantage of COD entirely. Consider that case of $\X=\Y$, the procedure of COD will reduce to FD, but the error bound of (\ref{bound:COD}) becomes a special case of (\ref{bound:FD}) in Lemma \ref{lem:FD} with $k=0$. The real-world dataset typically enjoys some approximate low-rank structure, which leads to the right-hand side of bound (\ref{bound:FD}) could be much 
smaller than the one of (\ref{bound:COD}). Hence although COD has better empirical performance, the existing error bounds are not tight enough. 

\begin{algorithm}[ht] 
	\caption{Co-Occurring Directions (COD)}\label{alg:COD}
	\begin{algorithmic}[1]
	    \STATE \textbf{Input:} $\X\in\BR^{n\times d_x}$, $\Y\in\BR^{n\times d_y}$ and sketch size $m$\\[0.1cm]
	    \STATE $\A\gets\vzero_{2m\times d_x}$ \\[0.1cm] 
	    \STATE $\B\gets\vzero_{2m\times d_y}$ \\[0.1cm]
		\STATE \textbf{for} $t=1,2,\dots,n$ \\[0.1cm]
		\STATE\label{line:inesert2A}\quad insert $(\vx^\ut)^\top$ into a zero valued row of $\A$ \\[0.1cm]
		\STATE\label{line:inesert2B}\quad insert $(\vy^\ut)^\top$ into a zero valued row of $\B$ \\[0.1cm]
		\STATE\quad \textbf{if} $\A$ or $\B$ has no zero valued rows \textbf{then} \\[0.1cm]
        \STATE\label{line:cod1}\quad\quad $(\Q_x,\R_x)\gets \QR\left(\A^\top\right)$\\[0.1cm]
		\STATE\quad\quad $(\Q_y,\R_y)\gets \QR\left(\B^\top\right)$\\[0.1cm]
		\STATE\quad\quad $[\U, \mSigma, \V]\gets \SVD(\R_x\R_y^\top)$ \\[0.1cm]
		\STATE\quad\quad $\delta \gets \sigma_m(\R_x\R_y^\top)$ \\[0.1cm]
        \STATE\quad\quad $\hSigma \gets \max\left(\mSigma-\delta\I_{2m},\vzero_{2m\times2m}\right)$ \\[0.1cm]
        \STATE\quad\quad $\A\gets\hSigma^{1/2}\U^\top\Q_x^\top$ \\[0.1cm]
		\STATE\label{line:cod2}\quad\quad $\B\gets\hSigma^{1/2}\V^\top\Q_y^\top$ \\[0.1cm]
        \STATE\quad \textbf{end if} \\[0.1cm]
		\STATE \textbf{end for} \\[0.1cm]
		\STATE \textbf{Output:} $\A$ and $\B$ \\[0.1cm]
	\end{algorithmic}
\end{algorithm}

\section{Sharper Analysis for COD}\label{section:error-bound}

In this section, we provide a tighter error bound for COD. We let $\delta^\ut$ be the value of $\delta$ at time step $t$. If
the algorithm does not enter the ``then'' section in the $t$-th step, then we have $\delta^\ut = 0$. Similarly, let $\A^\ut$, $\B^\ut$, $\Q_{x}^\ut$, $\Q_{y}^\ut$, $\U^\ut$, $\mSigma^\ut$, $\V^\ut$ and $\hSigma^\ut$ be the  corresponding variables after the main loop has been executed for $t$ times. Additionally, we use $\hA^\ut$ and $\hB^\ut$ to represent the matrices after insert operations (line \ref{line:inesert2A}-\ref{line:inesert2B}) have been executed at the $t$-th iteration. We need the following two lemmas for  proving our main results.
\begin{lem}[\citealt{mroueh2016co}]\label{lem:bound:sigma} 
The output matrices $\A$ and $\B$ of Algorithm \ref{alg:COD} satisfy
\begin{align}\label{boudn:delta1}
    \norm{\X^\top\Y-\A^\top\B} \leq \sum_{t=1}^n \delta^\ut 
\end{align}
and
\begin{align}\label{boudn:delta2}
    \big\|\A^\top\B\big\|_* \leq \big\|\X\big\|_F\big\|\Y\big\|_F - m\sum_{t=1}^n \delta^\ut. 
\end{align}
\end{lem}

\begin{lem}\label{lem:bound:nuclear}
The output of Algorithm~\ref{alg:COD} holds that
{\small\begin{align}\label{ieq:bound:nuclear}
\big\|\X^\top\Y\big\|_*-\big\|\A^\top\B\big\|_* 
\leq \sum_{i=k+1}^d \sigma_i(\X^\top\Y) +  k\sum_{t=1}^n\delta^\ut.
\end{align}}
\end{lem}

%Lemma \ref{lem:bound:sigma} follows the standard analysis of COD~\cite{mroueh2016co}, 
Lemma \ref{lem:bound:nuclear} is the key lemma of our proof. It improves the result in analysis of COD~\cite{mroueh2016co}. The term $\sum_{i=k+1}^d \sigma_i(\X^\top\Y)$ on the right-hand side of (\ref{ieq:bound:nuclear}) considers the potential approximate low-rank structure of $\X^\top\Y$, which leads to a tighter error bound of COD as follows.   

\begin{thm}\label{thm:dense-error}
The output of Algorithm~\ref{alg:COD} holds that
{\begin{align*}
    \norm{\X^\top\Y-\A^\top\B}
\leq \frac{1}{m-k}\Big(\Norm{\X}_F\Norm{\Y}_F-\big\|\X^\top\Y\big\|_k\Big).
\end{align*}}
for any $k<m$.
\end{thm}
\begin{proof}
Let $\Delta=\sum_{t=1}^n\delta^\ut$.
Connecting inequality (\ref{boudn:delta2}) in Lemma \ref{lem:bound:sigma} and the result of Lemma \ref{lem:bound:nuclear}, we have
\begin{align*}
  & m\Delta + \Norm{\X^\top\Y}_* - \Norm{\X}_F\Norm{\Y}_F \\
\leq & \Norm{\X^\top\Y}_*  -  \Norm{\A^\top\B}_*
\leq \sum_{i=k+1}^d \sigma_i(\X^\top\Y) +  k\Delta,
\end{align*}
that is 
$\Delta \leq \frac{1}{m-k}\left(\Norm{\X}_F\Norm{\Y}_F  -  \Norm{\X^\top\Y}_k\right)$.
Substituting above bound of $\Delta$ into inequality (\ref{boudn:delta2}) of Lemma \ref{lem:bound:sigma}, 
we finish the proof of this theorem.
\end{proof}

To achieve the accuracy that $\norm{\X^\top\Y-\A^\top\B}\leq\eps$, the previous error bound (Lemma~\ref{lem:COD}) requires the sketch size to be at least $m_1=\frac{1}{\eps}\Norm{\X}_F\Norm{\Y}_F$,
while Theorem~\ref{thm:dense-error} only requires the sketch size
$m_2=k+\frac{1}{\eps}\left(\Norm{\X}_F\Norm{\Y}_F-\Norm{\X^\top\Y}_k\right)$.
When input matrices $\X$ and $\Y$ have strongly correlation and approximate low-rank structure, $m_2$ could be much smaller than $m_1$.

In addition, the error bound of Theorem~\ref{thm:dense-error} matches that of FD (Lemma \ref{lem:FD}) when $\X=\Y$:
{\small\begin{align*}
  &  \norm{\X^\top\X-\A^\top\A} 
\leq  \frac{1}{m-k}\big(\Norm{\X}_F\Norm{\X}_F  -  \sum_{i=1}^k \sigma_i(\X^\top\X)\big) \\ 
& =  \frac{1}{m-k}\big(\Norm{\X}_F^2  -  \sum_{i=1}^k \sigma_i^2(\X)\big) 
=  \frac{1}{m-k}\left(\Norm{\X}_F^2-\Norm{\X_k}_F^2\right)\!.
\end{align*}}
On the other hand, the previous error bound (Lemma~\ref{lem:COD}) is worse than that of FD (Lemma~\ref{lem:FD}) in the symmetric case of $\X=\Y$. %It is a special case of (\ref{bound:FD}) with $k=0$. 

\section{Space Lower Bounds Analysis}\label{section:lower-bound}

In this section, we show that COD is space optimal with respect to our new error bound in Theorem \ref{thm:dense-error}. We first introduce the following lemma for low-rank matrices.

\begin{lem}[\citealt{kapralov2013differentially}]\label{lem:exp}
    For each $\delta > 0$ there exits a set of matrices $\fQ = \{\Q_1, \cdots, \Q_N\}$ and $N = 2^{\Omega(\ell(d-\ell) \log(1/\delta))}$, where $\Q_i \in \BR^{\ell \times d}$ with $\Q_i \Q_i^{\top} = \I_\ell$, such that $\norm{\Q_i \Q_j^{\top}} < 1 - \delta$.
\end{lem}

By using Lemma $\ref{lem:exp}$, we can construct a sets contains exponential number of matrices that each pair of them are not ``too close''. The formalized result is shown in Lemma \ref{lem:exp-distance}.
\begin{lem}\label{lem:exp-distance}
    For each $\delta > 0$ and $d_x \le d_y$ there exits a set of matrices  $\widehat\fZ_\ell = \{(\hX^{(1)}, \hY^{(1)}), \cdots, (\hX^{(N)}, \hY^{(N)})\}$, where $N = 2^{\Omega(\ell(d_y-\ell) \log(1/\delta))}$ and $\hX^\ui \in \BR^{\ell \times d_x}, \hY^\ui \in \BR^{\ell \times d_y}$ satisfy $\hX^\ui \hX^{\ui\top} = \I_\ell~\text{and}~\hY^\ui \hY^{\ui\top} = \I_{\ell}$
    for any $i=1,\dots,n$ and 
    \begin{align*}
    \big\|\hX^{\ui\top}\hY^{\ui\top} - \hX^{\uj\top}\hY^{\uj\top}\big\|_2 > \sqrt{2\delta} 
    \end{align*}
    for any $j\neq i$.
\end{lem}
\begin{proof}
Based on Lemma \ref{lem:exp}, there exist a set of matrices
$\fY = \{\hY^{(1)}, \cdots, \hY^{(N)}\}$, 
where $N = 2^{\Omega(\ell(d-\ell) \log(1/\delta))}$; and $\hY^\ui \in \BR^{\ell \times d}$ satisfies $\hY^\ui\hY^{\ui\top}=\I_\ell$ and $\big\|\hY^\ui \hY^{\uj\top}\big\| < 1 - \delta$.
We further set $\hX^\ui = [\I_\ell, {\bf 0}_{\ell\times(d_x-\ell)}]$.
We have
{\begin{align*}
 & \big\|\hX^{\ui\top}\hY^\ui - \hX^{\uj\top}\hY^\uj\big\|_2^2 
=  \big\|\hY^\ui - \hY^\uj\big\|_2^2 \\
= & \big\|(\hY^\ui - \hY^\uj)(\hY^{\ui\top} - \hY^{\uj\top})\big\|_2 \\
\geq & 2 - \big\|\hY^\uj \hY^{\ui\top} + \hY^\ui \hY^{\uj\top}\big\|_2 
\geq  2 \delta,
\end{align*}}
where we use the definition of $\hX^\ui$, $\hY^\ui$ and the fact $\norm{\A^\top\A}=\norm{\A}^2$.
\end{proof}

Then we present a lower bound of space complexity for approximate matrix multiplication, which matches the memory cost of COD. Hence, we can conclude that COD is space optimal with respect to the guaranteed accuracy in Theorem~\ref{thm:dense-error}.
\begin{thm}\label{thm:lower-bound}
    We consider any matrix sketching algorithm with inputs as $\X\in\BR^{n\times d_x}$ and $\Y\in\BR^{n\times d_y}$ and outputs $\A\in\BR^{m\times d_x}$ and $\B\in\BR^{m\times d_y}$ with guarantee
    {\begin{align*}
        \norm{\X^\top \Y-\A^\top \B} \leq \frac{1}{m-k}\Big(\Norm{\X}_F\Norm{\Y}_F-\big\|\X^\top\Y\big\|_k\Big)
    \end{align*}}
    for any $k<m$. Assuming that a constant number of bits is required to describe a word (i.e., a unit of memory), then the algorithm requires at least $\Omega(m(d_x+d_y))$ bits of space.
\end{thm}
\begin{proof}
    Without loss of generality, we suppose that $d_y\geq d_x$. 
    Let $\widehat\fZ_\ell = \{(\hX^{(1)}, \hY^{(1)}), \cdots, (\hX^{(N)}, \hY^{(N)})\}$ be the set of matrices defined in Lemma~\ref{lem:exp-distance} with $\ell=m/4$, $\delta=1/8$ and $N=2^{\Omega(\frac{m}{4}\cdot(d_y-m/4)\log(8))}$.
    We construct matrices 
    $\X^\ui=[\hX^\ui; {\bf 0}_{(n-m/4)\times d_x}]\in\BR^{n\times d_x}$ 
    and
    $\Y^\ui=[\hY^\ui; {\bf 0}_{(n-m/4)\times d_y}]\in\BR^{n\times d_y}$
    for $i=1,\dots,N$.  
    Then we have $\fZ_\ell=\{(\X^\ui,\Y^\ui)\}_{i=1}^N$ which satisfies
    $\norm{\X^{\ui\top} \Y^\ui - \X^{\uj\top} \Y^\uj} > 1/2$
    for each $i\neq j$.
    Let $(\A,\B)$ be the output of the matrix sketching algorithm with input $(\X^\ui,\Y^\ui)$.
    The guarantee of the algorithm indicates
    \begin{align*}
        \big\|\X^{\ui\top}\Y^\ui-\A^\top \B\big\|_2 \leq \frac{1}{m-k}\left(\frac{m}{4}-k\right) \leq \frac{1}{4}.
    \end{align*}
    Hence, each $(\A,\B)$ only encodes one matrix pencil in $\fZ_\ell$ (the product of the matrices), which means that the lower bound of space complexity to attach the desired accuracy is 
    $\log_2N=\Omega(md_y)=\Omega(m(d_x+d_y))$ bits.
\end{proof}

\section{Sparse Co-Occurring Directions}\label{section:SCOD}

In this section, we proposed a variant of COD for sparse AMM. We also prove its error bound is similar to our improved result of COD.

\subsection{The Algorithm}

We describe details of our sparse co-occurring directions (SCOD) in Algorithm \ref{alg:SCOD}. 
The procedure of SCOD maintains the sparse data in two buffer matrices $\X'$ and $\Y'$. 
The algorithm restricts the non-zero entries in buffers to be less than $m(d_x+d_y)$ and the number of row of each buffer is at most $d_x+d_y$. When the buffers are full, we perform subspace power method (SPM)~\cite{woodruff14sketching,musco2015randomized}  to approximate the data in the buffers by low-rank matrices $\tX\in\BR^{m \times d_x}$ and $\tY\in\BR^{m \times d_y}$ such that $\X'^\top\Y'\approx\tX^\top\tY$. We present the procedure of SPM in Algorithm~\ref{alg:SPM}. 
%The line \ref{line:balance} in Algorithm \ref{alg:SCOD} balances the magnitude of $\tX$ and $\tY$, which is crucial to our theoretical analysis. 

Let $\tX^\ui$ and $\tY^\ui$ be the results of $\tX$ and $\tY$ after Algorithm~\ref{alg:SCOD} has executed ``then'' section for $i$-times. Define 
$\C=[\tX^{(1)};\cdots;\tX^{(T)}]$ and
$\D=[\tY^{(1)};\cdots;\tY^{(T)}]$
where $T$ is the number of total times we enter ``then'' section of the algorithm. Then $\C$ and $\D$ are the estimators of $\X$ and $\Y$ respectively and the procedure of SCOD can be regarded as as running standard COD on input matrices $\C$ and $\D$ in streaming fashion.
Since the row numbers of buffers $\X'$ and $\Y'$ could be much larger than $m$, the operations on dense matrices (line \ref{line:start-dense}-\ref{line:spd2}) will not be executed frequently. Hence, SCOD is much more efficient than COD for sparse inputs. 

\begin{algorithm}[ht] 
	\caption{Subspace Power Method (SPM)}\label{alg:SPM}
	\begin{algorithmic}[1]
	    \STATE \textbf{Input:} $\M\in\BR^{d_1\times d_2}$, target rank $m$ and integer $q>0$ \\[0.1cm]
	    \STATE $\G=[G_{ij}]\in\BR^{d_2\times m}$, where $G_{ij}\sim\fN(0,1)$ i.i.d \\[0.1cm]
	    \STATE\label{line:spm-multiply} $\K=\left(\M\M^\top\right)^q\M\G\in\BR^{d_1\times m}$ \\[0.1cm]
	    \STATE $\Z\gets$ orthonormal column basis of $\K$ \\[0.1cm]
		\STATE \textbf{Output:} $\Z$ \\[0.1cm]
	\end{algorithmic}
\end{algorithm}

\begin{algorithm}[ht] 
	\caption{Sparse Co-Occurring Directions (SCOD)}\label{alg:SCOD}
	\begin{algorithmic}[1]
	    \STATE \textbf{Input:} $\X\in\BR^{n\times d_x}$, $\Y\in\BR^{n\times d_y}$, sketch size $m$, 
	    failure probability $\delta$ and sequence $\{q_i\}_{i=1,2\dots}$ \\[0.1cm]
	    \STATE $i=0$  \\[0.1cm]
	    \STATE $\A\gets\vzero_{m\times d_x}$,~  $\B\gets\vzero_{m\times d_y}$ \\[0.1cm]
	    \STATE $\X'\gets{\rm empty}$,~$\Y'\gets{\rm empty}$ \\[0.1cm]
		\STATE \textbf{for} $t=1,2,\dots,n$ \\[0.1cm]
		\STATE\quad $\X'\gets [\X';(\vx^\ut)^\top]$,~$\Y'\gets [\Y';(\vy^\ut)^\top]$ \\[0.1cm]
		\STATE\label{line:if}\quad \textbf{if } $\nnz(\X')+\nnz(\Y')>m(d_x+d_y)$ \textbf{or} $t=n$ \\[0.1cm] \quad\quad \textbf{or} ${\rm rows}(\X')=d_x+d_y$ \textbf{or} ${\rm rows}(\Y')=d_x+d_y$ \\[0.1cm]
		\quad\textbf{then} \\[0.1cm]
		\STATE\quad\quad\quad\label{line:SPM} $\Z=\SPM(\X'^\top\Y',m,q_i)$ \\[0.1cm]
		\STATE\quad\quad\quad\label{line:balance} $[\tU,\tSigma,\tV]=\SVD\left(\Z^\top\X'^\top\Y'\right)$ \\[0.1cm]
		\STATE\label{line:XY}\quad\quad\quad $\tX\gets\tSigma^{1/2}\tU^\top\Z^\top$  \\[0.1cm]
		\STATE\label{line:XY2}\quad\quad\quad $\tY\gets\tSigma^{1/2}\tV^\top$ \\[0.1cm]
		\STATE\label{line:spd1}\quad\quad\quad $\A \gets [\A; \tX]$ \\[0.1cm]
		\STATE\quad\quad\quad $\B \gets [\B; \tY]$ \\[0.1cm]
        \STATE\label{line:start-dense}\quad\quad\quad $(\Q_x,\R_x)\gets \QR\left(\A^\top\right)$\\[0.1cm]
		\STATE\quad\quad\quad $(\Q_y,\R_y)\gets \QR\left(\B^\top\right)$\\[0.1cm]
		\STATE\quad\quad\quad $[\U, \mSigma, \V]\gets \SVD(\R_x\R_y^\top)$ \\[0.1cm]
		\STATE\quad\quad\quad $\delta \gets \sigma_m(\R_x\R_y^\top)$ \\[0.1cm]
        \STATE\quad\quad\quad $\hSigma \gets \max\left(\mSigma-\delta\I_{m},\vzero_{m\times m}\right)$ \\[0.1cm]
        \STATE\quad\quad\quad $\A\gets\hSigma^{1/2}\U^\top\Q_x^\top$ \\[0.1cm]
		\STATE\label{line:spd2}\quad\quad\quad $\B\gets\hSigma^{1/2}\V^\top\Q_y^\top$ \\[0.1cm]
	    \STATE\quad\quad\quad $\X'\gets{\rm empty}$,~~$\Y'\gets{\rm empty}$ \\[0.1cm]
	    \STATE\quad\quad\quad $i\gets i+1$ \\[0.1cm]
        \STATE\quad \textbf{end if} \\[0.1cm]
		\STATE \textbf{end for} \\[0.1cm]
		\STATE \textbf{Output:} $\A$ and $\B$ \\[0.1cm]
	\end{algorithmic}
\end{algorithm}

\subsection{Analysis of Error Bound}

The analysis of SCOD is more challenging than sparse frequent directions (SFD)~\cite{ghashami2016efficient} which only addresses the case of $\X=\Y$. The reason is the ``mergeability property'' of FD~\cite{ghashami2016frequent,desai2016improved,ghashami2016efficient} only works for Frobenius norm and it is not applicable to COD.

The approximation error of SCOD comes from two parts: the compressing error from sub-routine SPM and the merge error from estimating $\C^\top\D$ by $\A^\top\B$. 
We first consider a single call of SPM, which approximation error can be bounded as follows.

\begin{lem}\label{lem:randSVD}
Let $q=\tilde\Theta(\log(md_1/p)/\eps)$ for Algorithm \ref{alg:SPM}, then the output $\Z$ satisfies
$\norm{\M - \Z\Z^\top\M} \leq (1+\eps)\sigma_{m+1}(\M)$
with probability at least $1-p$.
\end{lem}

Based on Lemma \ref{lem:randSVD}, we can bound the total compressing error of SOCD by the following lemma.

\begin{lem}\label{lem:compress}
Setting $q_i=\tilde\Theta(\log(md_1/p_i)/\eps)$ and $p_i=\delta/2i^2$, then we have then Algorithm \ref{alg:SCOD} holds that
\begin{align*}
\norm{\X^\top\Y-\C^\top\D} 
\leq  \dfrac{1+\eps}{m-k}\left(\Norm{\X}_F\Norm{\Y}_F - \big\|\X^\top\Y\big\|_k\right),
\end{align*}
for any $k<m$ and $\eps>0$ with probability $1-\delta$. 
\end{lem}
\begin{proof}
Let $\X'^\ui$ and $\Y'^\ui$ be the value of $\X'$ and $\Y'$ when we execute subspace power methods at $i$-th time in line \ref{line:SPM} of Algorithm~\ref{alg:SCOD}, then we have 
\begin{align*}
\X=[\X'^{(1)};\dots;\X'^{(T)}] \text{~~and~~} \Y=[\Y'^{(1)};\dots;\Y'^{(T)}].
\end{align*}
Using Lemma \ref{lem:randSVD} with $\M=\tX^{\ui\top}\tY^\ui$ and $q=q_i$, then with probability $1-p_i$, we have 
{\small\begin{align}
     & \norm{\tX^{\ui\top}\tY^\ui - \X'^{\ui\top}\Y'^\ui}
\leq  (1+\eps)\sigma_{m+1}\left(\X'^{\ui\top}\Y'^\ui\right) \nonumber\\
\leq & \frac{1+\eps}{m-k}\left(\big\|\X'^\ui\big\|_F\big\|\Y'^\ui\big\|_F-\big\|\X'^{\ui\top}\Y'^\ui\big\|_k\right), \label{ieq:sp-compress}
\end{align}}
where the last step use~\citeauthor{srebro2005maximum}'s (\citeyear{srebro2005maximum}) Lemma 1 such that
$\|\X'^{\ui\top}\Y'^{\ui}\|_*
\leq \|\X'^{\ui}\|_F\|\Y'^{\ui}\|_F$.
%$\Norm{\M^\top\N}_*\leq\Norm{\M}_F\Norm{\N}_F$.

Summing over inequality (\ref{ieq:sp-compress}) with $i=1,\dots,T$, we have
{\small\begin{align*}
 & \big\|\X^\top\Y-\C^\top\D\big\|_2  
\leq  \sum_{t=1}^T\big\|\X'^\ui\Y'^\ui - \tX^{\ui\top}\tY^\ui\big\|_2  \\
\leq & \frac{1+\eps}{m-k}\sum_{t=1}^T\left(\big\|\X'^\ui\big\|_F\big\|\Y'^\ui\big\|_F-\big\|\X'^{\ui\top}\Y'^\ui\big\|_k\right) \\
\leq & \frac{1+\eps}{m-k}\left(\big\|\X\big\|_F\big\|\Y\big\|_F-\big\|\X^\top\Y\big\|_k\right)
\end{align*}}
with probability $1-\delta$. The last inequality is based on the Cauchy–Schwarz inequality and the triangle inequality of Ky Fan $k$-norm. Note that the failure probability is no more than $p_1+\dots+p_T=\frac{\delta}{2}\sum_{i=1}^T1/i^2\leq\delta$.
\end{proof}
Unlike  SFD~\cite{ghashami2016frequent} which introduces a verifying step to boost the success probability, our method instead requires $q_i$ to be increased  logarithmically to ensure the error bound of SCOD holds with probability at least $1-\delta$ for given $\delta\in(0,1)$.
Another important property of SCOD is that the compression step shrink the magnitude of the product of input matrices. The steps in line~\ref{line:XY}-\ref{line:XY2} of Algorithm~\ref{alg:SCOD}  balance the singular values of $\tX$ and $\tY$, which leads to the following lemma:
\begin{lem}\label{lem:fro-shrink}
Algorithm \ref{alg:SCOD} holds that
\begin{align*}
\big\|\tX^\ui\big\|_F\big\|\tY^\ui\big\|_F \leq \big\|\X'^\ui\big\|_F\big\|\Y'^\ui\big\|_F.
\end{align*}
\end{lem}
Since the analysis of merging error is similar to standard COD, we can establish the error bound of SCOD by using above lemmas.
\begin{thm}\label{thm:sparse-error}
Setting $q_i=\tilde\Theta(\log(md_1/p_i)/\eps)$ with constant $\eps>0$ and $p_i=\delta/2i^2$, with probability $1-\delta$, the outputs $\A$ and $\B$ of Algorithm~\ref{alg:SCOD} hold that 
{\begin{align*}
  & \norm{\X^\top\Y-\A^\top\B} \\
\leq & \left(\frac{2+\eps}{m-k}+\frac{(1+\eps)k}{(m-k)^2}\right)
\left(\Norm{\X}_F\Norm{\Y}_F -\big\|\X^\top\Y\big\|_k\right)
\end{align*}}
for all $k<m$.
\end{thm}
\begin{proof}
Consider that $\A$ and $\B$ can be viewed as the output of running Algorithm~\ref{alg:COD} with input matrices
\begin{align*}
    \C=[\tX^{(1)};\cdots;\tX^{(T)}] \text{~~and~~} \D=[\tY^{(1)};\cdots;\tY^{(T)}].
\end{align*} 
Following the proof of Theorem~\ref{thm:dense-error}, we have
{\small\begin{align*}
\begin{split}
  & \norm{\C^\top\D-\A^\top\bf\B} \\
\leq & \frac{1}{m-k}\left(\sum_{i=1}^T\big\|\tX^\ui\big\|_F\big\|\tY^\ui\big\|_F -\big\|\C^\top\D\big\|_k\right) \\
\leq & \frac{\sum_{i=1}^T\big\|\tX^\ui\big\|_F\big\|\tY^\ui\big\|_F - \big\|\X^\top\Y\big\|_k + \big\|\X^\top\Y-\C^\top\D\big\|_k}{m-k} \\
\leq & \frac{\sum_{i=1}^T\big\|\tX^\ui\big\|_F\big\|\tY^\ui\big\|_F - \big\|\X^\top\Y\big\|_k + k\big\|\X^\top\Y-\C^\top\D\big\|_2}{m-k} \\
\leq & \frac{1}{m-k}\left(\sum_{i=1}^T\big\|\X'^\ui\big\|_F\big\|\Y'^\ui\big\|_F - \big\|\X^\top\Y\big\|_k\right) \\
 &\quad + \frac{(1+\eps)k}{(m-k)^2}\left(\Norm{\X}_F\Norm{\Y}_F-\big\|\X^\top\Y\big\|_k\right) \\
\leq & \left(\frac{1}{m-k} + \frac{(1+\eps)k}{(m-k)^2}\right)\left(\Norm{\X}_F\Norm{\Y}_F-\big\|\X^\top\Y\big\|_k\right)
\end{split}
\end{align*}}
where we use Lemma~\ref{lem:compress},~\ref{lem:fro-shrink} and triangle inequality.

Combing above results and Lemma~\ref{lem:compress}, we have
{\small\begin{align*}
 & \norm{\X^\top\Y-\A^\top\B} \\
\leq & \norm{\X^\top\Y-\C^\top\D} + \norm{\C^\top\D-\A^\top\B} \\
%\leq & \color{red} (wrong) \frac{1+(1+\eps)(k+1)}{m-k}\left(\sum_{t=1}^p\Norm{\X_t}_F\Norm{\Y_t}_F -\Norm{\X^\top\Y}_k\right) \\
\leq & \left(\frac{2+\eps}{m-k}+\frac{(1+\eps)k}{(m-k)^2}\right)
\left(\Norm{\X}_F\Norm{\Y}_F -\big\|\X^\top\Y\big\|_k\right),
\end{align*}}
with probability at least $1-\delta$. 
\end{proof}

\subsection{Complexity Analysis}

We use the constant-word-size model for our analysis like that of sparse FD~\cite{ghashami2016frequent}. We suppose floating point numbers are represented by a constant number of bits, random access into memory requires $\fO(1)$ time and multiplying a sparse matrix $\M$ by a dense vector requires $\fO(\nnz(\M))$ time and storing $\M$ requires $\fO(\nnz(\M))$ space. 

The procedure of SCOD (Algorithm~\ref{alg:SCOD}) implies the buffer $\X'$ and $\Y'$ is sparse and contains at most $m(d_x+d_y)$ non-zero entries and it is not difficult to verify that all dense matrices in the algorithm cost no more than $\fO(m(d_x+d_y))$ space. 
Hence, the space complexity of SCOD is $\fO(m(d_x+d_y))$ in total which is the same as COD~(Algorithm \ref{alg:COD}).

Then we analyze the time complexity of SCOD.
The constraints on buffer size means we have
\begin{align*}
    T \leq \frac{\nnz(\X)+\nnz(\Y)}{m(d_x+d_y)} + \frac{n}{d_x+d_y}.
\end{align*}

Since each QR factorization or SVD on $m\times d$ matrix cost $\fO(m^2d)$ time, the operation on dense matrices of Algorithm~\ref{alg:SCOD} from line \ref{line:balance}-\ref{line:spd2} requires at most
\begin{align*}
    \fO(m^2(d_x+d_y)T) = \fO(m(\nnz(\X)+\nnz(\Y))+m^2n).
\end{align*}

Note that SCOD calls SPM with input $\M=\X'^\top\Y'$. Since both $\X'$ and $\Y'$ are sparse, it is unnecessary to construct $\M$ explicitly and we can multiply $\X'$ and $\Y'$ on $\G$ separately in line \ref{line:spm-multiply} of Algorithm~\ref{alg:SPM}. Then the time complexity of executing SPM needs
$\fO(mq_i(\nnz(\X'^\ui)+\nnz(\Y'^\ui))+m^2d_x)$
when the algorithm enters ``then'' section at the $i$-th time. Following the upper bound of $T$ and the setting of $q_i$ in Theorem~\ref{thm:sparse-error}, the calls of SPM in Algorithm~\ref{alg:SPM} entirely takes at most
{\small\begin{align*}
 & \fO\left(\sum_{i=1}^T\left(mq_i(\nnz(\X'^\ui)+\nnz(\Y'^\ui))+m^2d_x\right)\right) \\
\leq &\fO\left(mq_T(\nnz(\X)+\nnz(\Y))+Tm^2d_x\right) \\
%\leq &\fO\big(m(\nnz(\X) + \nnz(\Y))\log(2T^2md_x/\delta)/\eps \\
%     &\quad\quad + m(\nnz(\X)+\nnz(\Y))+m^2n\big) \\
%\leq &\fO\big(m(\nnz(\X) + \nnz(\Y))\log(Tmd_x/\delta) 
%     +m^2n\big) \\
= & \tilde\fO\big(m(\nnz(\X) + \nnz(\Y)) + m^2n\big).
\end{align*}}
Hence, the total time complexity of proposed SCOD is $\tilde\fO\big(m(\nnz(\X) + \nnz(\Y)) + m^2n\big)$.% which is much more efficient than COD when $\X$ and $\Y$ are very sparse.

\section{Numerical Experiments}\label{section:experiments}

In this section, we empirically compare the proposed  sparse co-occurring directions (SCOD) with frequent direction based AMM (FD-AMM)~\cite{ye2016frequent}, co-occurring directions (COD)~\cite{mroueh2016co} and sparse frequent direction based AMM algorithm (SFD-AMM)\footnote{SFD-AMM refers to the method simply replacing FD step in FD-AMM with sparse frequent directions~\cite{ghashami2016frequent}. We provide more detailed discussion about SFD-AMM in appendix.}. 
Instead of increasing $q_i$ logarithmically as the analysis of Theorem~\ref{thm:lower-bound}, we fix $q_i=5$ in our experiment since the empirical error arise from subspace power method is very small in practice.

We evaluate performance of all algorithms on cross-language datasets: Amazon Product Reviews (APR), PAN-PC-11 (PAN), JRC Acquis (JRC) and Europarl (EURO) which contain millions of English (EN), French (FR) and Spanish (ES) sentences~\cite{prettenhofer2010cross,potthast2010evaluation,potthast2011cross,koehn2005europarl}. We use bag-of-words feature for our experiments. All of input matrices are large but very sparse and we summary the parameters in Table~\ref{table:dataset}.

We demonstrate sketch-error and time-error comparisons in Figure~\ref{figure:sketch-error} and~\ref{figure:time-error} respectively. It is apparently that SCOD always performs better than all baseline algorithms. 
We do not include the curve of FD-AMM and COD in time-error comparison because these two algorithms take much more time than others. Due to the limit of space, we defer the result of sketch-time comparison and detailed computing infrastructure in appendix. 

\section{Conclusion}\label{section:conclusion}

In this paper, we first improved the error bound of a deterministic sketching algorithm COD for streaming AMM problem.
In symmetric case, our  result matches the error bound of classical algorithm FD.
We also proved COD matches the space lower bound complexity to achieve our error bound. 
In addition, we proposed a sparse variant of COD  with a reasonable error bound. 
The experimental results show that the proposed algorithm  has better performance than baseline methods in practice. 

It would be interesting to borrow the idea of this paper to establish better theoretical guarantees and streaming algorithms for more classical machine learning and statistical models such as canonical correlation analysis~\cite{hotelling1992relations,avron2013efficient,ye2016frequent}, generalized eigenvector  decomposition~\cite{bhatia2018gen,gene1996matrix} and spectral co-clustering~\cite{dhillon2001co}.

\begin{table*}[ht]
    \centering\renewcommand{\arraystretch}{1.2}
    \vskip0.2cm
    {\begin{tabular}{|c|ccc|cc|}
        \hline
        Dataset & $n$ & $d_x$ & $d_y$ & ${\rm density}(\X)$ & ${\rm density}(\Y)$ \\\hline
        APR (EN-FR) & $2.32 \times 10^4$ & $2.80 \times 10^4$ & $4.28 \times 10^4$ & $6.31 \times 10^{-4}$ & $4.53 \times 10^{-4}$ \\
        PAN (EN-FR) & $8.90 \times 10^4$ & $5.12 \times 10^4$ & $9.96 \times 10^4$ & $4.38 \times 10^{-4}$ & $2.43 \times 10^{-4}$ \\
        JRC (EN-FR) & $1.50 \times 10^5$ & $1.72 \times 10^5$ & $1.87 \times 10^5$ & $1.65 \times 10^{-4}$ & $1.64 \times 10^{-4}$ \\
        JRC (EN-ES) & $1.50 \times 10^5$ & $1.72 \times 10^5$ & $1.92 \times 10^5$ & $1.65 \times 10^{-4}$ & $1.60 \times 10^{-4}$ \\
        JRC (FR-ES) & $1.50 \times 10^5$ & $1.87 \times 10^5$ & $1.92 \times 10^5$ & $1.64 \times 10^{-4}$ & $1.60 \times 10^{-4}$ \\
        EURO (EN-FR) & $4.76 \times 10^5$ & $7.25 \times 10^4$ & $8.77 \times 10^4$ & $3.46 \times 10^{-4}$ & $3.65 \times 10^{-4}$\\
        EURO (EN-ES) & $4.76 \times 10^5$ & $7.25 \times 10^4$ & $8.80 \times 10^4$ & $3.46 \times 10^{-4}$ & $3.47 \times 10^{-4}$ \\
        EURO (FR-ES) & $4.76 \times 10^5$ & $8.77 \times 10^4$ & $8.80 \times 10^4$ & $3.65 \times 10^{-4}$ & $3.47 \times 10^{-4}$ \\\hline
    \end{tabular}}
    \caption{We present the size and density of datasets used in our experiments, where ${\rm density}(\X)=\nnz(\X)/nd_x$ and ${\rm density}(\Y)=\nnz(\Y)/nd_y$. All of these datasets are publicly available~\cite{ferrero2016multilingual}.}\label{table:dataset}
\end{table*}

\begin{figure*}[ht]
\centering
\begin{tabular}{cccc}
     \includegraphics[scale=0.30]{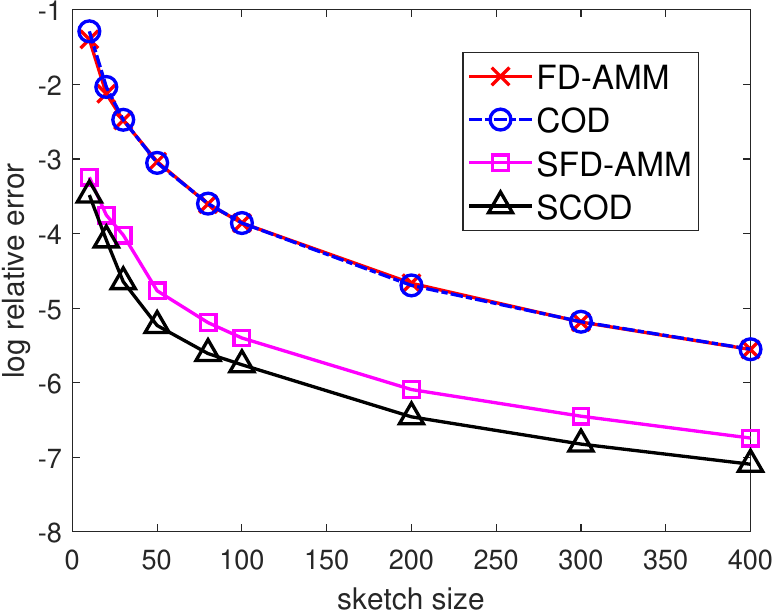} &
     \includegraphics[scale=0.30]{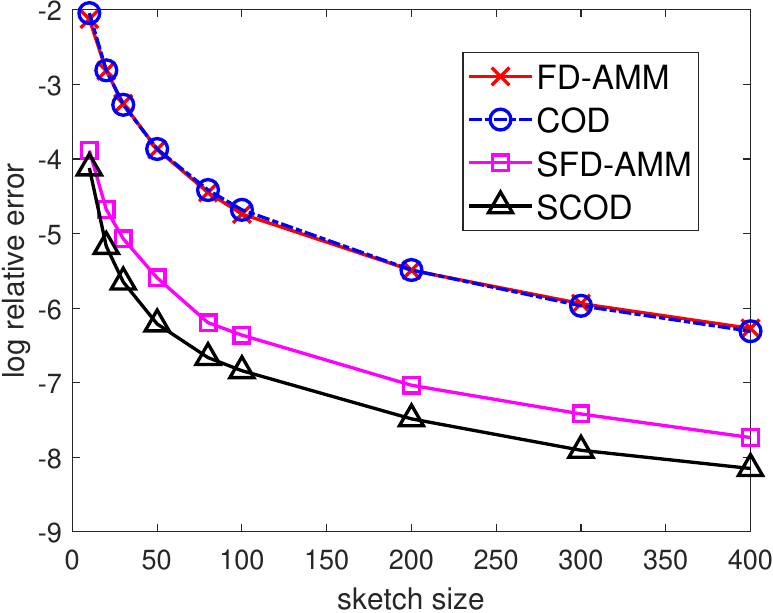} &
     \includegraphics[scale=0.30]{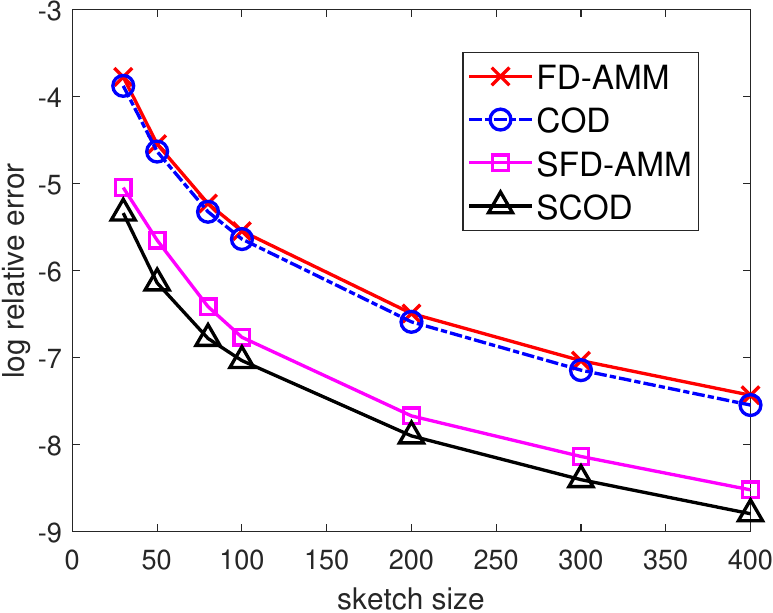} &a
     \includegraphics[scale=0.30]{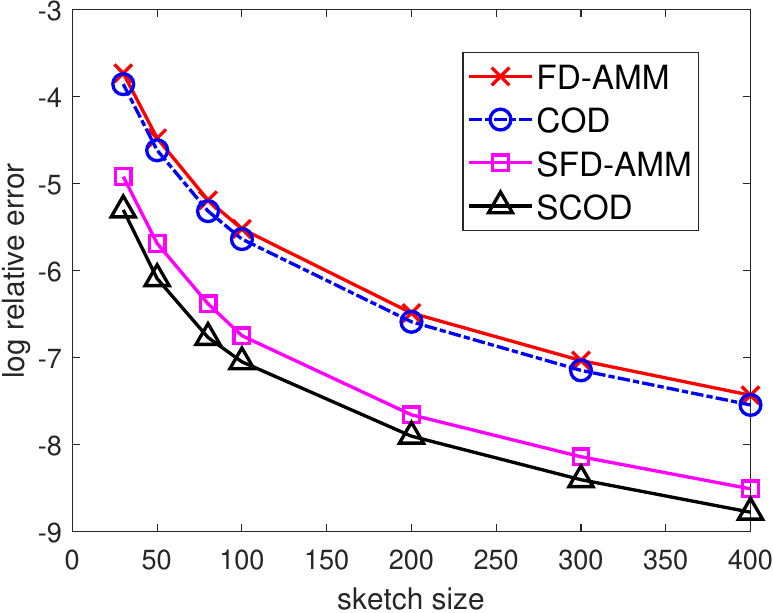} \\
     \small (a) APR (EN-FR) & 
     \small (b) PAN (EN-FR) & 
     \small (d) JRC (EN-FR) &
     \small (e) JRC (FR-ES) \\[0.1cm]
     \includegraphics[scale=0.30]{figure/jrc0-a.pdf} &
     \includegraphics[scale=0.30]{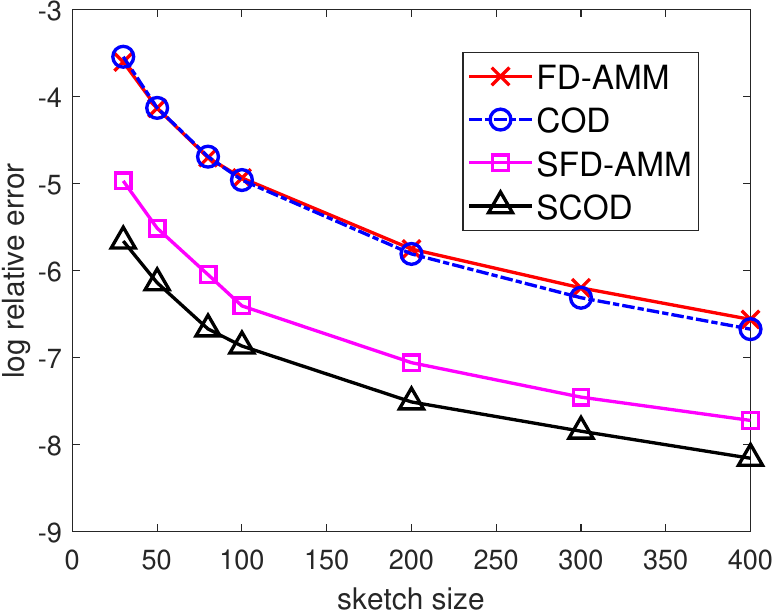} &
     \includegraphics[scale=0.30]{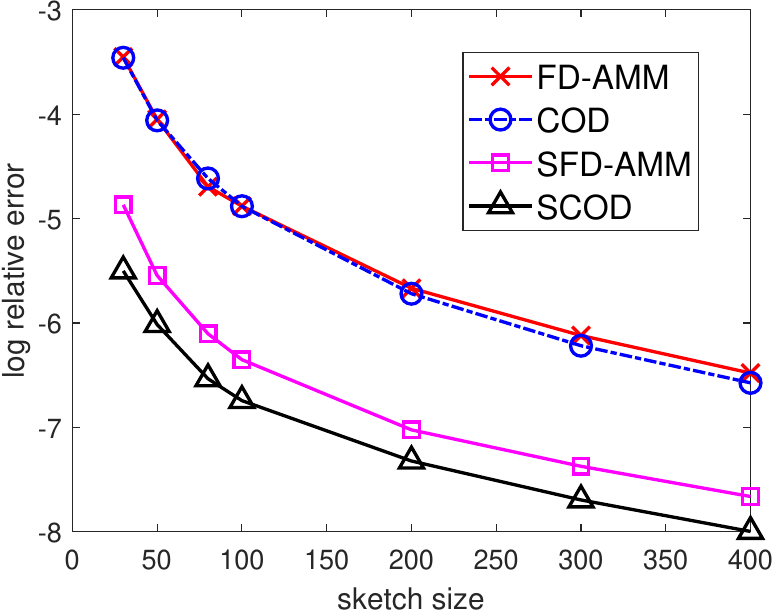} &
     \includegraphics[scale=0.30]{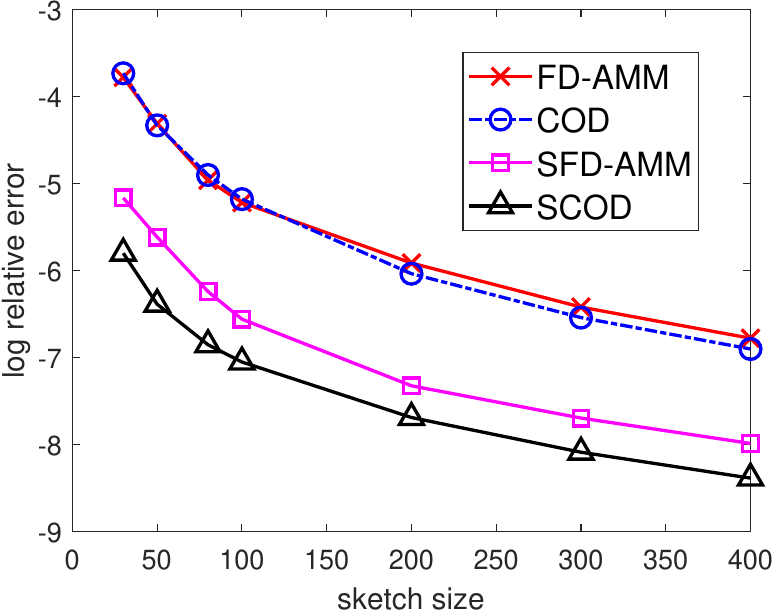} \\
     \small (f) JRC (FR-ES) & 
     \small (g) EURO (EN-FR) & 
     \small (j) EURO (EN-ES) &
     \small (i) EURO (FR-ES)
\end{tabular}
\caption{The plot of sketch size against relative spectral norm error}\label{figure:sketch-error}
\end{figure*}

\begin{figure*}[!ht]
\centering
\begin{tabular}{cccc}
     \includegraphics[scale=0.299]{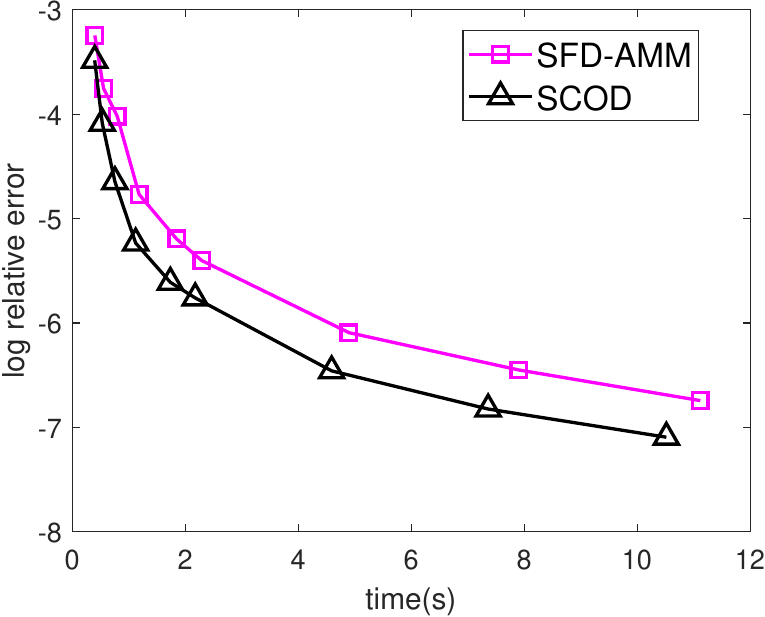} &
     \includegraphics[scale=0.299]{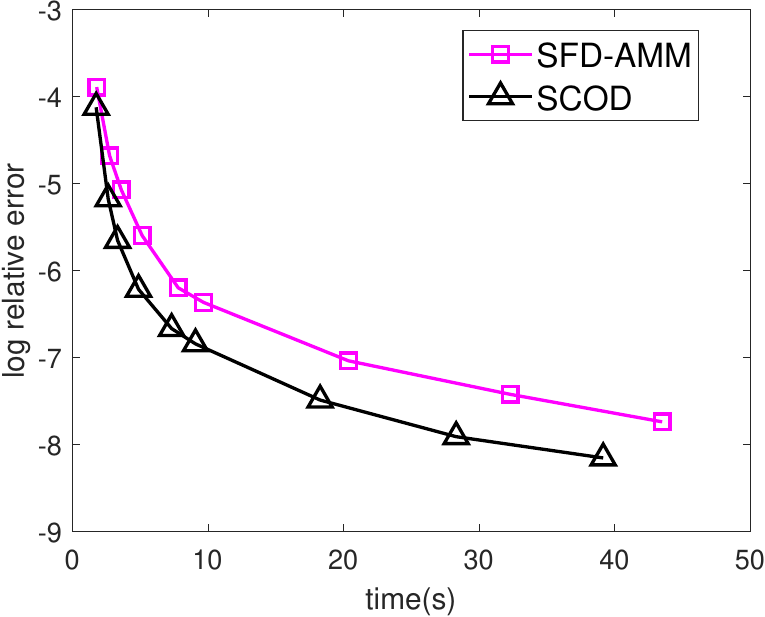} &
     \includegraphics[scale=0.299]{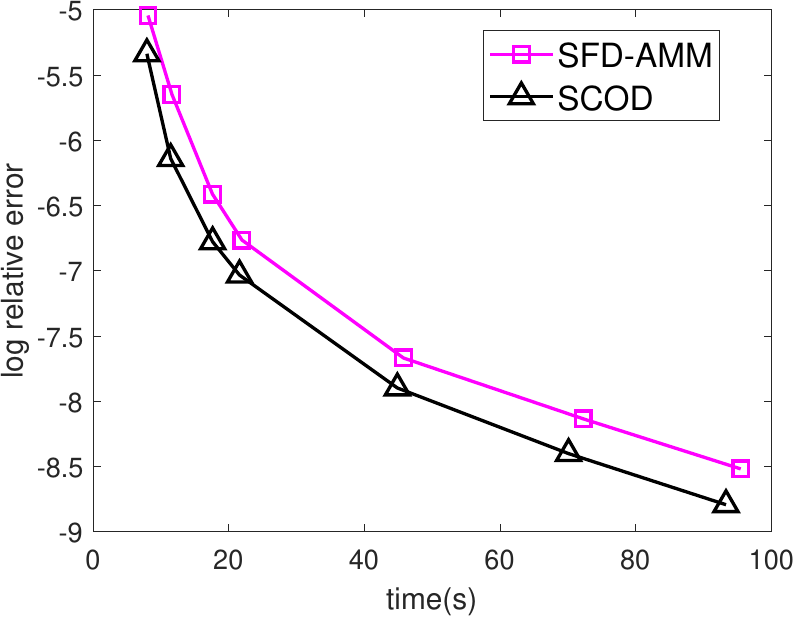} &
     \includegraphics[scale=0.299]{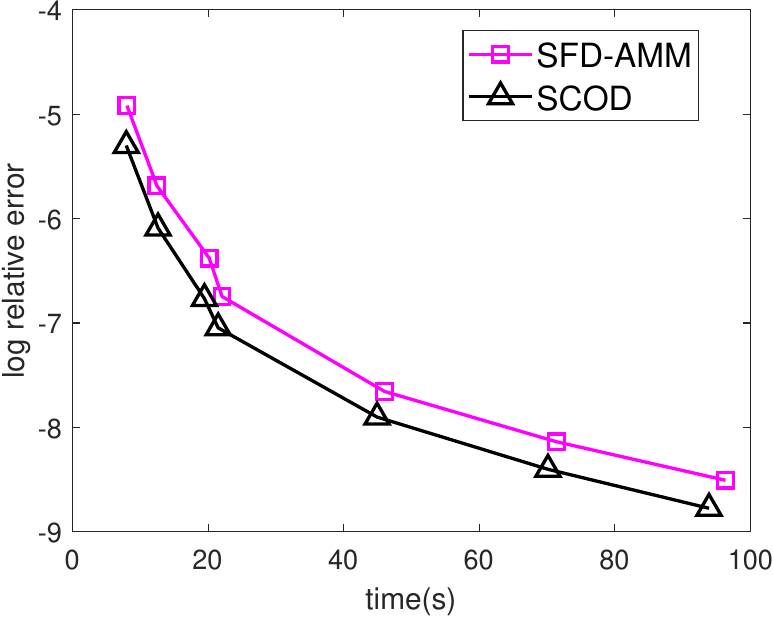} \\
     \small (a) APR (EN-FR) & 
     \small (b) PAN (EN-FR) & 
     \small (d) JRC (EN-FR) &
     \small (e) JRC (FR-ES) \\[0.1cm]
     \includegraphics[scale=0.299]{figure/jrc0-c.pdf} &
     \includegraphics[scale=0.299]{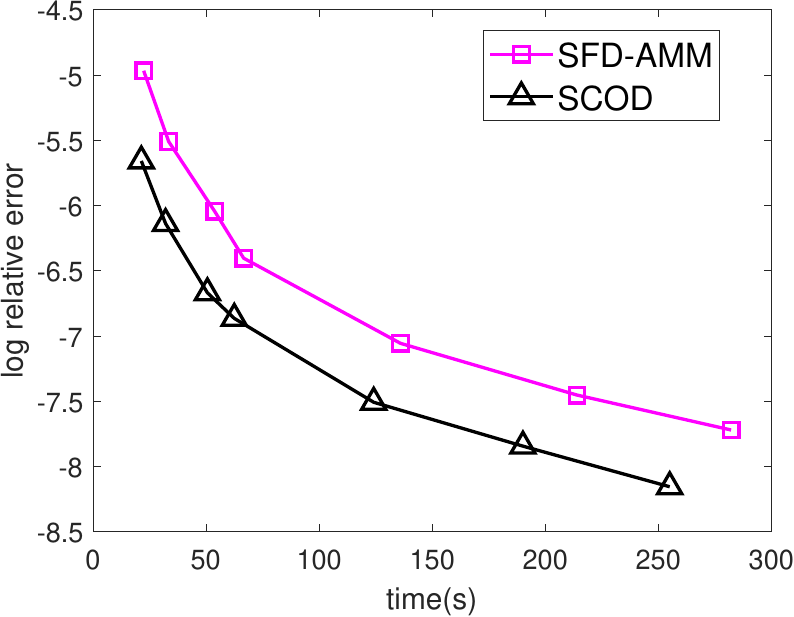} &
     \includegraphics[scale=0.299]{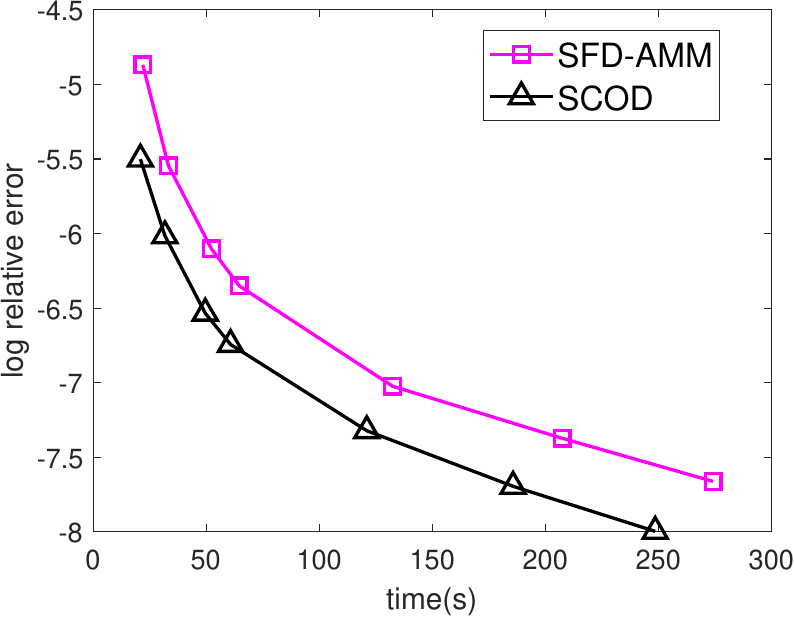} &
     \includegraphics[scale=0.299]{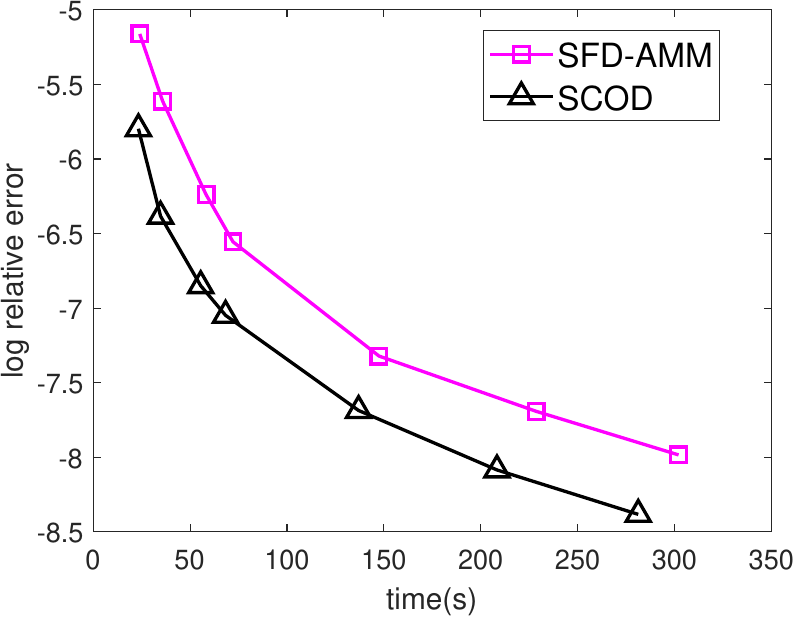} \\
     \small (f) JRC (FR-ES) & 
     \small (g) EURO (EN-FR) & 
     \small (j) EURO (EN-ES) &
     \small (i) EURO (FR-ES)
\end{tabular}
\caption{The plot of time (s) against relative spectral norm error}\label{figure:time-error}
\end{figure*}

\section*{Acknowledgements}
Luo Luo is supported by GRF 16201320. Haishan Ye is supported by Shenzhen Research Institute of Big Data (named “Automated Machine Learning”). %Guangzeng Xie is supported by ......

\bibliography{reference}

\begin{thebibliography}{39}
\providecommand{\natexlab}[1]{#1}
\providecommand{\url}[1]{\texttt{#1}}
\providecommand{\urlprefix}{URL }
\expandafter\ifx\csname urlstyle\endcsname\relax
  \providecommand{\doi}[1]{doi:\discretionary{}{}{}#1}\else
  \providecommand{\doi}{doi:\discretionary{}{}{}\begingroup
  \urlstyle{rm}\Url}\fi

\bibitem[{Avron et~al.(2013)Avron, Boutsidis, Toledo, and
  Zouzias}]{avron2013efficient}
Avron, H.; Boutsidis, C.; Toledo, S.; and Zouzias, A. 2013.
\newblock Efficient dimensionality reduction for canonical correlation
  analysis.
\newblock In \emph{ICML}.

\bibitem[{Bhatia et~al.(2018)Bhatia, Pacchiano, Flammarion, Bartlett, and
  Jordan}]{bhatia2018gen}
Bhatia, K.; Pacchiano, A.; Flammarion, N.; Bartlett, P.~L.; and Jordan, M.~I.
  2018.
\newblock Gen-{Oja}: Simple \& efficient algorithm for streaming generalized
  eigenvector computation.
\newblock In \emph{NIPS}.

\bibitem[{Chen et~al.(2020)Chen, Luo, Zhang, Yu, and Lian}]{chen2020bandit}
Chen, C.; Luo, L.; Zhang, W.; Yu, Y.; and Lian, Y. 2020.
\newblock Efficient and Robust High-Dimensional Linear Contextual Bandits.
\newblock In \emph{IJCAI}.

\bibitem[{Clarkson and Woodruff(2017)}]{clarkson2017low}
Clarkson, K.~L.; and Woodruff, D.~P. 2017.
\newblock Low-rank approximation and regression in input sparsity time.
\newblock \emph{Journal of the ACM} 63(6): 1--45.

\bibitem[{Desai, Ghashami, and Phillips(2016)}]{desai2016improved}
Desai, A.; Ghashami, M.; and Phillips, J.~M. 2016.
\newblock Improved practical matrix sketching with guarantees.
\newblock \emph{IEEE Transactions on Knowledge and Data Engineering} 28(7):
  1678--1690.

\bibitem[{Dhillon(2001)}]{dhillon2001co}
Dhillon, I.~S. 2001.
\newblock Co-clustering documents and words using bipartite spectral graph
  partitioning.
\newblock In \emph{SIGKDD}.

\bibitem[{Dickens(2020)}]{dickens2020ridge}
Dickens, C. 2020.
\newblock Ridge Regression with Frequent Directions: Statistical and
  Optimization Perspectives.
\newblock \emph{arXiv preprint:2011.03607} .

\bibitem[{Drineas, Kannan, and Mahoney(2006)}]{drineas2006fast}
Drineas, P.; Kannan, R.; and Mahoney, M.~W. 2006.
\newblock Fast Monte Carlo algorithms for matrices {I}: Approximating matrix
  multiplication.
\newblock \emph{SIAM Journal on Computing} 36(1): 132--157.

\bibitem[{Ferrero et~al.(2016)Ferrero, Agnes, Besacier, and
  Schwab}]{ferrero2016multilingual}
Ferrero, J.; Agnes, F.; Besacier, L.; and Schwab, D. 2016.
\newblock A multilingual, multi-style and multi-granularity dataset for
  cross-language textual similarity detection.
\newblock In \emph{LREC}.

\bibitem[{Ghashami, Liberty, and Phillips(2016)}]{ghashami2016efficient}
Ghashami, M.; Liberty, E.; and Phillips, J.~M. 2016.
\newblock Efficient frequent directions algorithm for sparse matrices.
\newblock In \emph{SIGKDD}.

\bibitem[{Ghashami et~al.(2016)Ghashami, Liberty, Phillips, and
  Woodruff}]{ghashami2016frequent}
Ghashami, M.; Liberty, E.; Phillips, J.~M.; and Woodruff, D.~P. 2016.
\newblock Frequent directions: Simple and deterministic matrix sketching.
\newblock \emph{SIAM Journal on Computing} 45(5): 1762--1792.

\bibitem[{Ghashami and Phillips(2014)}]{ghashami2014relative}
Ghashami, M.; and Phillips, J.~M. 2014.
\newblock Relative errors for deterministic low-rank matrix approximations.
\newblock In \emph{SODA}.

\bibitem[{Golub and Loan(1996)}]{gene1996matrix}
Golub, G.~H.; and Loan, C. F.~V. 1996.
\newblock Matrix computations.
\newblock \emph{Johns Hopkins Universtiy Press, 3rd edtion} .

\bibitem[{Horn and Johnson(1994)}]{horn1994topics}
Horn, R.~A.; and Johnson, C.~R. 1994.
\newblock \emph{Topics in matrix analysis}.
\newblock Cambridge university press.

\bibitem[{Hotelling(1992)}]{hotelling1992relations}
Hotelling, H. 1992.
\newblock Relations between two sets of variates.
\newblock In \emph{Breakthroughs in statistics}, 162--190. Springer.

\bibitem[{Huang(2019)}]{huang2019near}
Huang, Z. 2019.
\newblock Near optimal frequent directions for sketching dense and sparse
  matrices.
\newblock \emph{Journal of Machine Learning Research} 20(56): 1--23.

\bibitem[{Kapralov and Talwar(2013)}]{kapralov2013differentially}
Kapralov, M.; and Talwar, K. 2013.
\newblock On differentially private low rank approximation.
\newblock In \emph{SODA}.

\bibitem[{Koehn(2005)}]{koehn2005europarl}
Koehn, P. 2005.
\newblock Europarl: A parallel corpus for statistical machine translation.
\newblock In \emph{MT summit}, volume~5, 79--86. Citeseer.

\bibitem[{Kuzborskij, Cella, and Cesa-Bianchi(2019)}]{kuzborskij2019efficient}
Kuzborskij, I.; Cella, L.; and Cesa-Bianchi, N. 2019.
\newblock Efficient linear bandits through matrix sketching.
\newblock In \emph{AISTATS}.

\bibitem[{Liberty(2013)}]{liberty2013simple}
Liberty, E. 2013.
\newblock Simple and deterministic matrix sketching.
\newblock In \emph{SGIKDD}.

\bibitem[{Luo et~al.(2016)Luo, Agarwal, Cesa-Bianchi, and
  Langford}]{luo2016efficient}
Luo, H.; Agarwal, A.; Cesa-Bianchi, N.; and Langford, J. 2016.
\newblock Efficient second order online learning by sketching.
\newblock In \emph{NIPS}.

\bibitem[{Luo et~al.(2019)Luo, Chen, Zhang, Li, and Zhang}]{luo2019robust}
Luo, L.; Chen, C.; Zhang, Z.; Li, W.-J.; and Zhang, T. 2019.
\newblock Robust Frequent Directions with Application in Online Learning.
\newblock \emph{Journal of Machine Learning Research} 20(45): 1--41.

\bibitem[{Luo et~al.(2018)Luo, Zhang, Zhang, Zhu, Zhang, and
  Pei}]{luo2018sketched}
Luo, L.; Zhang, W.; Zhang, Z.; Zhu, W.; Zhang, T.; and Pei, J. 2018.
\newblock Sketched follow-the-regularized-leader for online factorization
  machine.
\newblock In \emph{SIGKDD}.

\bibitem[{Martinsson et~al.(2010)Martinsson, Szlam, Tygert
  et~al.}]{martinsson2010normalized}
Martinsson, P.-G.; Szlam, A.; Tygert, M.; et~al. 2010.
\newblock Normalized power iterations for the computation of {SVD}.
\newblock In \emph{NIPS Workshop on Low-rank Methods for Large-scale Machine
  Learning}.

\bibitem[{Misra and Gries(1982)}]{misra1982finding}
Misra, J.; and Gries, D. 1982.
\newblock Finding repeated elements.
\newblock \emph{Science of computer programming} 2(2): 143--152.

\bibitem[{Mroueh, Marcheret, and Goel(2017)}]{mroueh2016co}
Mroueh, Y.; Marcheret, E.; and Goel, V. 2017.
\newblock Co-{O}ccuring directions sketching for approximate matrix multiply.
\newblock In \emph{AISTATS}.

\bibitem[{Musco and Musco(2015)}]{musco2015randomized}
Musco, C.; and Musco, C. 2015.
\newblock Randomized block {K}rylov methods for stronger and faster approximate
  singular value decomposition.
\newblock In \emph{NIPS}.

\bibitem[{Potthast et~al.(2011)Potthast, Barr{\'o}n-Cede{\~n}o, Stein, and
  Rosso}]{potthast2011cross}
Potthast, M.; Barr{\'o}n-Cede{\~n}o, A.; Stein, B.; and Rosso, P. 2011.
\newblock Cross-language plagiarism detection.
\newblock \emph{Language Resources and Evaluation} 45(1): 45--62.

\bibitem[{Potthast et~al.(2010)Potthast, Stein, Barr{\'o}n-Cede{\~n}o, and
  Rosso}]{potthast2010evaluation}
Potthast, M.; Stein, B.; Barr{\'o}n-Cede{\~n}o, A.; and Rosso, P. 2010.
\newblock An evaluation framework for plagiarism detection.
\newblock In \emph{COLING}.

\bibitem[{Prettenhofer and Stein(2010)}]{prettenhofer2010cross}
Prettenhofer, P.; and Stein, B. 2010.
\newblock Cross-language text classification using structural correspondence
  learning.
\newblock In \emph{ACL}.

\bibitem[{Rudelson and Vershynin(2010)}]{rudelson2010non}
Rudelson, M.; and Vershynin, R. 2010.
\newblock Non-asymptotic theory of random matrices: extreme singular values.
\newblock In \emph{ICM}.

\bibitem[{Sarlos(2006)}]{sarlos2006improved}
Sarlos, T. 2006.
\newblock Improved approximation algorithms for large matrices via random
  projections.
\newblock In \emph{FOCS}.

\bibitem[{Shi and Phillips(2020)}]{shi2020deterministic}
Shi, B.; and Phillips, J.~M. 2020.
\newblock A deterministic streaming sketch for ridge regression.
\newblock \emph{arXiv preprint:2002.02013} .

\bibitem[{Srebro, Rennie, and Jaakkola(2005)}]{srebro2005maximum}
Srebro, N.; Rennie, J.; and Jaakkola, T.~S. 2005.
\newblock Maximum-margin matrix factorization.
\newblock In \emph{NIPS}.

\bibitem[{Wan and Zhang(2020)}]{wan2020approximate}
Wan, Y.; and Zhang, L. 2020.
\newblock Approximate Multiplication of Sparse Matrices with Limited Space.
\newblock \emph{arXiv preprint:2009.03527} .

\bibitem[{Wegelin(2000)}]{wegelin2000survey}
Wegelin, J.~A. 2000.
\newblock A survey of Partial Least Squares ({PLS}) methods, with emphasis on
  the two-block case.
\newblock \emph{University of Washington, Technical Report} .

\bibitem[{Woodruff(2014{\natexlab{a}})}]{woodruff2014low}
Woodruff, D.~P. 2014{\natexlab{a}}.
\newblock Low rank approximation lower bounds in row-update streams.
\newblock In \emph{NIPS}.

\bibitem[{Woodruff(2014{\natexlab{b}})}]{woodruff14sketching}
Woodruff, D.~P. 2014{\natexlab{b}}.
\newblock Sketching as a Tool for Numerical Linear Algebra.
\newblock \emph{Foundations and Trends in Theoretical Computer Science}
  10(1-2): 1--157.

\bibitem[{Ye, Luo, and Zhang(2016)}]{ye2016frequent}
Ye, Q.; Luo, L.; and Zhang, Z. 2016.
\newblock Frequent direction algorithms for approximate matrix multiplication
  with applications in {CCA}.
\newblock In \emph{IJCAI}.

\end{thebibliography}

\clearpage\newpage\appendix\onecolumn
%\begin{center}\textbf{\LARGE Supplementary Materials}\\[1cm]\end{center}

In this supplementary materials, Section~\ref{appdendix:proof-sigma}-\ref{appdendix:proof-balance} provide detailed proofs of lemmas we used in main text. Section~\ref{appendix:experiment} gives more details of experiments. We also provide additional discussion on algorithm SFD-AMM in Section~\ref{appendix:SFD-AMM}. 

\section{The Proof of Lemma \ref{lem:bound:sigma}}\label{appdendix:proof-sigma}
This lemma can be proved by the analysis of \citeauthor{mroueh2016co}'s~(\citeyear{mroueh2016co})~Theorem 2. We reformulate the details by our notations here for completeness.

\begin{proof}
Let $\hA^\ut=(\mSigma^\ut)^{1/2}\U^{\ut\top}\Q_x^{\ut\top}$ and $\hB^\ut=(\mSigma^\ut)^{1/2}\V^{\ut\top}\Q_y^{\ut\top}$.
The first inequality can be proved as follows:
\begin{align*}
  & \norm{\X^\top\Y-\A^\top\B} \\
= & \norm{\sum_{t=1}^n\vx^\ut(\vy^\ut)^\top-\sum_{t=1}^n\left(\big(\A^{(t)}\big)^\top\B^{(t)}-\big(\A^{(t-1)}\big)^\top\B^{(t-1)}\right)} \\
%= & \norm{\sum_{t=1}^n\left(\vx^\ut(\vy^\ut)^\top-\big(\A^{(t)}\big)^\top\B^{(t)}+\big(\A^{(t-1)}\big)^\top\B^{(t-1)}\right)} \\
= & \norm{\sum_{t=1}^n\vx^\ut(\vy^\ut)^\top-\sum_{t=1}^n\left(\big(\A^{(t)}\big)^\top\B^{(t)}-\big(\A^{(t-1)}\big)^\top\B^{(t-1)}\right)} \\
= & \norm{\sum_{t=1}^n\left(\big(\hA^{(t)}\big)^\top\hB^{(t)}-\big(\A^{(t)}\big)^\top\B^{(t)}\right)} \\
\leq & \sum_{t=1}^n\norm{\big(\hA^{(t)}\big)^\top\hB^{(t)}-\big(\A^{(t)}\big)^\top\B^{(t)}} \\
\leq & \sum_{t=1}^n\norm{\Q_{x}^\ut\U^\ut\left({\bf\Sigma}^\ut-\delta^{\ut}\I_{2m}-{\bf\Sigma}^\ut\right)(\V^\ut)^\top(\Q_{y}^\ut)^\top}  \\
= & \sum_{t=1}^n \delta^{\ut},
\end{align*}
where we use triangle inequality and the definition of the notations.

Then we show the second inequality. Similar to above analysis, we have
\begin{align*}
  & \big\|\A^\top\B\big\|_* \\
%= &  \Norm{(\A^{(n)})^\top\B^{(n)}}_* \\ 
= &  \sum_{t=1}^n \left(\big\|(\A^\ut)^\top\B^\ut\big\|_* - \big\|(\A^\utm)^\top\B^\utm\big\|_*\right) \\
= & \sum_{t=1}^n \left(\big\|(\hA^\ut)^\top\hB^\ut\big\|_* - \big\|(\A^\utm)^\top\B^\utm\big\|_*\right) 
- \sum_{t=1}^n \left(\big\|(\hA^\ut)^\top\hB^\ut\big\|_* - \big\|(\A^\ut)^\top\B^\ut\big\|_*\right).
\end{align*}
The QR steps means $(\hA^\ut)^\top\hB^\ut$ and $(\A^\ut)^\top\B^\ut$ can be written as
\begin{align*}
(\hA^\ut)^\top\hB^\ut = \Q_x^\ut\U^\ut{\bf\Sigma}^\ut(\V^\ut)^\top\Q_y^\ut 
\text{~~and~~}
(\A^\ut)^\top\B^\ut = \Q_x^\ut\U^\ut({\bf\Sigma}-\sigma_m^\ut\I_m)^\ut(\V^\ut)^\top\Q_y^\ut.
\end{align*}
which implies
\begin{align*}
    \big\|(\hA^\ut)^\top\hB^\ut\big\|_* - \big\|(\A^\ut)^\top\B^\ut\big\|_* = m\sigma_m^\ut.
\end{align*}
Using triangle inequality, we have
\begin{align*}
  & \big\|(\hA^\ut)^\top\hB^\ut\big\|_* - \big\|(\A^\utm)^\top\B^\utm\big\|_* \\
\leq & \big\|(\hA^\ut)^\top\hB^\ut - (\A^\utm)^\top\B^\utm\big\|_* \\
= & \big\|\vx^\ut(\vy^\ut)^\top\big\|_* 
= \big\|\vx^\ut\big\|_2\big\|\vy^\ut\big\|_2
\end{align*}
Combing all above results, we have
\begin{align*}
   \Norm{\A^\top\B}_* 
\leq   \sum_{t=1}^n \big\|\vx^\ut\big\|_2\big\|\vy^\ut\big\|_2
- m\sum_{t=1}^n \sigma_m^\ut  
\leq   \Norm{\X}_F\Norm{\Y}_F - m\sum_{t=1}^n \sigma_m^\ut.
\end{align*}
\end{proof}

\section{The Proof of Lemma \ref{lem:bound:nuclear}}\label{appdendix:proof-nuclear}
This lemma is crucial to establish the tighter bound of COD. We first introduce the following property of Ky Fan $k$-norm.
\begin{lem}[{\citealt[Theorem 3.4.1]{horn1994topics}}]\label{lem:max-Fan-k}
Given matrix $\M$, we have
\begin{align*}
    \Norm{\M}_k=\max\left\{\left|\tr\left(\P^\top\M\Q\right)\right|: \P^\top\P=\I_k, \Q^\top\Q=\I_k\right\},
\end{align*}
where $\tr(\cdot)$ is the trace of the matrix.
\end{lem}

Then we prove Lemma \ref{lem:bound:nuclear} by using Lemma \ref{lem:max-Fan-k}.
\begin{proof}
We let $(\hP,\hQ)=\argmax\left\{\left|\tr\left(\P^\top\A\Q\right)\right|: \P^\top\P=\I, \Q^\top\Q=\I\right\}$.
Then we have
{\begin{align*}
  & \Norm{\X^\top\Y}_*-\Norm{\A^\top\B}_* \\
= & \sum_{i=1}^k \sigma_i(\X^\top\Y) + \sum_{i=k+1}^d \sigma_i(\X^\top\Y) -\Norm{\A^\top\B}_* \\   
= & \max\left\{\left|\tr\left(\P^\top(\X^\top\Y)\Q\right)\right| : \P^\top\P=\I_k, \Q^\top\Q=\I_k\right\} + \sum_{i=k+1}^d \sigma_i(\X^\top\Y) \\ 
& - \max\left\{\left|\tr\left(\P^\top(\A^\top\B)\Q\right)\right| : \P^\top\P=\I_k, \Q^\top\Q=\I_k\right\} - \sum_{i=k+1}^d \sigma_i(\A^\top\B)  \\ 
\leq &  \sum_{i=k+1}^d \sigma_i(\X^\top\Y) + \left|\tr\left(\hP^\top\left(\X^\top\Y\right)\hQ\right)\right|  - \left|\tr\left(\hP^\top\left(\A^\top\B\right)\hQ\right)\right|  \\
\leq & \sum_{i=k+1}^d \sigma_i(\X^\top\Y) + \left|\tr\left(\hP^\top\left(\X^\top\Y-\A^\top\B\right)\hQ\right)\right| \\
\leq & \sum_{i=k+1}^d \sigma_i(\X^\top\Y) +  \max\left\{\left|\tr\left(\P^\top\left(\X^\top\Y-\A^\top\B\right)\Q\right)\right|:\P^\top\P=\I_k, \Q^\top\Q=\I_k\right\} \\
\leq & \sum_{i=k+1}^d \sigma_i(\X^\top\Y) +  \max\left\{\left|\tr\left(\P^\top\left(\sum_{t=1}^n\sigma_m^\ut\Q_{x}^\ut\U^\ut\V^{\ut\top}\Q_y^{\ut\top}\right)\Q\right)\right|:\P^\top\P=\I_k, \Q^\top\Q=\I_k\right\}
\\
\leq & \sum_{i=k+1}^d \sigma_i(\X^\top\Y) +  \sum_{t=1}^n\delta^\ut\max\left\{\left|\tr\left(\P^\top\left(\Q_{x}^\ut\U^\ut\V^{\ut\top}\Q_y^{\ut\top}\right)\Q\right)\right|:\P^\top\P=\I_k, \Q^\top\Q=\I_k\right\} \\
= & \sum_{i=k+1}^d \sigma_i(\X^\top\Y) +  \sum_{t=1}^n\delta^\ut\sum_{i=1}^k\sigma_i\left(\Q_{x}^\ut\U^\ut\V^{\ut\top}\Q_y^{\ut\top}\right) \\
= & \sum_{i=k+1}^d \sigma_i(\X^\top\Y) +  k\sum_{t=1}^n\delta^\ut,
\end{align*}}
where the first inequality is due to the definition of $\hP$ and $\hQ$; the second and the third one use triangle inequality; the last two inequality is based on the procedure of the algorithm; all equalities come from Lemma~\ref{lem:max-Fan-k} and the procedure of COD.
\end{proof}

\section{The Proof of Lemma \ref{lem:randSVD}}\label{appdendix:proof-spm}

We can prove Lemma \ref{lem:randSVD} by modifying the analysis in 
Section 4.3 of~\citeauthor{woodruff14sketching}'s~(\citeyear{woodruff14sketching}) survey. We present the details for completeness. The proof is based on the following lemma.

\begin{lem}[{\citealt[Proposition 2.4 and (3.2)]{rudelson2010non}}]\label{lem:largest-sv}
Let $\mOmega\in\BR^{d_1\times d_2}$ be a random matrix whose entries are independent mean zero sub-gaussian random variables whose subgaussian moments are bounded by 1. Then we have
\begin{enumerate}
    \item $\BP\big(\norm{\mOmega}>C(\sqrt{d_1}+\sqrt{d_2}+t)\big)\leq 2\exp(-ct^2)$ for any $t>0$;
    \item $\BP(\sigma_{\min}(\mOmega)\leq\zeta d^{-1/2})\leq\zeta$ when $d_1=d_2=d$ for any $\zeta>0$;
\end{enumerate}
where $c>0$ and $C>0$ are some constants.
\end{lem}
%Lemma \ref{lem:largest-sv} means to match success probability with $p$, we need $t=\sqrt{\frac{1}{c}\log(2/p)}$ and $\eps=p$ to guarantee 
%\begin{align}\label{bound:largest-sv}
%    \BP\big(\norm{\A}^2 \leq \fO(d_1+d_2+\log(1/p))\big) \geq \sqrt{1-p}.
%\end{align}
%and
%\begin{align}
%     \BP(\sigma_{\min}^2(\A)\geq p^2/d) \geq \sqrt{1-p}.
%\end{align}
Then we provide the proof of Lemma \ref{lem:randSVD}.
\begin{proof}
Let $\N=(\M\M^\top)^q\M$.
By \citeauthor{woodruff14sketching}'s (\citeyear{woodruff14sketching}) Lemma 4.14, $\Z\Z^\top\M$ is the best rank-$m$ approximation of $\M$
in the column space of $\Z$ with respect to the spectral norm. 
Hence, we have
\begin{align*}
    \norm{\M-\Z\Z^\top\M} 
\leq \norm{\M-(\Z\Z^\top \N)(\Z\Z^\top\N)^\dagger\M}
\leq \norm{\left(\I_{d_1}-(\Z\Z^\top \N)(\Z\Z^\top\N)^\dagger\right)\M},
\end{align*}
where the notation $(\cdot)^\dagger$ presents pseudo-inverse; the inequality follows $\Z\Z^\top\N$ is of rank-$m$ and in the column space of $\Z$.

Since $\I_{d_1}-(\Z\Z^\top\N)(\Z\Z^\top\N)^\dagger$
is a projection matrix, we can apply \citeauthor{woodruff14sketching}'s (\citeyear{woodruff14sketching}) Lemma 4.15 to infer that 
\begin{align*}
& \norm{(\I_{d_1}-(\Z\Z^\top\N)(\Z\Z^\top\N)^\dagger)\M} \\
\leq & \norm{(\I_{d_1}-(\Z\Z^\top\N)(\Z\Z^\top\N)^\dagger)(\M\M^\top)^q\M}^{1/(2q+1)} \\
= & \norm{\N-(\Z\Z^\top\N)(\Z\Z^\top\N)^\dagger\N}^{1/(2q+1)} \\
= & \norm{\N-\Z\Z^\top\N}^{1/(2q+1)}
\end{align*}
where we use that $(\Z\Z^\top\N)^\dagger = (\Z^\top\N)^\dagger\Z^\top$ since $\Z$ has orthonormal columns, and thus
\begin{align*}
    (\Z\Z^\top\N)(\Z\Z^\top\N)^\dagger\N 
= (\Z\Z^\top\N)(\Z^\top\N)^\dagger(\Z^\top\N) = \Z\Z^\top\N.
\end{align*}
Hence, we have
\begin{align}\label{bound:MNZq}
    \norm{\M - \Z\Z^T \M} \leq \norm{\N - \Z\Z^\top\N}^{1/(2q+1)}.
\end{align}
Let $\U\mSigma\V^\top$ be the SVD of $\N$, $\mOmega_U=\V_m^\top\G\in\BR^{m\times m}$ and $\mOmega_L=\V^\top_{d_1-m}\G\in\BR^{(d_1-m)\times m}$, where $\V^\top_m$ denotes the top $m$ rows of $\V^\top$ and $\V^\top_{d_1-m}$ the remaining rows. 
Since $\V^\top$ are column orthonormal, by rotational invariance of the Gaussian distribution, both $\mOmega_U$ and $\mOmega_L$ are independent matrices of i.i.d. $\fN(0, 1)$ entries.

We now apply \citeauthor{woodruff14sketching}'s (\citeyear{woodruff14sketching}) Lemma 4.4 with the $\C$ of that lemma equal to $\Z$ above, the $\Z$ of that lemma equal to $\V_m$, and the $\A$ of that lemma equal to $\N$ above. This implies the $\E$ of that lemma is equal to $\N-\N_m$. 
Note that to apply the lemma we need $\V^T_m\G$ to have full rank, which holds with probability 1 since it is a $m\times m$ matrix of i.i.d. $\fN(0, 1)$ random variables. We thus have
\begin{align}
\begin{split}\label{bound:MNZ}
  &  \norm{\N-\Z\Z^\top\N}^2 \\
= &  \norm{\N-\N_m}^2 + \norm{(\N-\N_m)\G(\V_m^\top\G)^\dagger}^2 \\
= &  \norm{\N-\N_m}^2 + \norm{\U_{d_1-m}\mSigma_{d_1-m}\V_{d_1-m}^\top\G(\V_m^\top\G)^\dagger}^2 \\
= &  \norm{\N-\N_m}^2 + \norm{\mSigma_{d_1-m}\V_{d_1-m}^\top\G(\V_m^\top\G)^\dagger}^2 \\
\leq & \norm{\N-\N_m}^2\left(1+\norm{\mOmega_L}^2\big\|\mOmega_U^\dagger\big\|^2\right)
\end{split}
\end{align}
where $\mSigma_{d_1-m}$ denotes the $(d_1-m)\times(d_1-m)$ diagonal matrix whose entries are the bottom $d_1-m$ diagonal entries of $\mSigma$, and $\U_{d_1-m}$ denotes the rightmost $d_1-m$ columns of $\U$. Here in the second equality we use unitary invariance of $\U_{d_1-m}$, while in the inequality we use sub-multiplicativity of the spectral norm. 

By using Lemma~\ref{lem:largest-sv} with 
$\mOmega=\mOmega_L$ and $t=\sqrt{c^{-1}\log(4/p)}$, we have
\begin{align}\label{bound:largest-sv1}
    \BP\left(\norm{\mOmega_L}^2 
\leq \left(\sqrt{d_1-m}+\sqrt{m} +\sqrt{c^{-1}\log\left(4/p\right)}\right)^2\right) \geq 1-\frac{p}{2}.
\end{align}
By using Lemma~\ref{lem:largest-sv} with $\mOmega=\mOmega_U$ and $\zeta=p/2$, we have
\begin{align}\label{bound:largest-sv2}
    \BP\left(\sigma_{\min}^2(\mOmega_U) \geq \frac{p^2}{4m}\right) 
\geq 1-\frac{p}{2}.
\end{align}
Since $\mOmega_L$ and $\mOmega_U$ are independent, combing inequalities (\ref{bound:largest-sv1}) and (\ref{bound:largest-sv2}), we have
\begin{align}\label{bound:sv3}
    1+\norm{\mOmega_L}^2\|\mOmega_U^\dagger\|^2
\leq & 1+\left(\sqrt{d_1-m}+\sqrt{m} +\sqrt{c^{-1}\log\left(4/p\right)}\right)^2\cdot\frac{4m}{p^2} 
\leq \frac{c_0(d_1+\log(1/p))m}{p^2}
\end{align}
for some constant $c_0>0$ with probability at least $(1-p/2)^2>1-p$.

Combining results of (\ref{bound:MNZq}), (\ref{bound:MNZ}) and (\ref{bound:sv3}), we have
\begin{align*}
    \norm{\M - \Z\Z^T \M} 
\leq \norm{\N - \N_m}^{1/(2q+1)}\cdot\left(\frac{c_0(d_1+\log(1/p))m}{p^2}\right)^{1/(4q+2)}.
\end{align*}
Noting that $\norm{\N - \N_m}=\norm{\M - \M_m}^{2q+1}$ and setting
\begin{align*}
q =\frac{1}{4}\left(\frac{1}{\eps}\log\left(\frac{c_0(d_1+\log(1/p))m}{p^2}\right)-2\right) 
= \tilde\Theta\left(\frac{1}{\eps}\log\left(\frac{md_1}{p}\right)\right)
\end{align*}
%\begin{align*}
%& x=\frac{c_0(d_1+\log(1/p))m}{p^2} \\
%& x^{1/(4q+2)} \leq 1+\eps \\
%\Longleftrightarrow & \frac{\log_{1+\eps}x}{4q+2} \leq 1 \\
%\Longleftrightarrow & \frac{\log x}{(4q+2)\log(1+\eps)} \leq 1 \\
%\Longleftarrow & \frac{\log x}{(4q+2)\eps} \leq 1  \\
% & q=\Theta(\log(x)/\eps)
%\end{align*}
we have
\begin{align*}
    \norm{\M - \Z\Z^T \M} 
\leq (1+\eps)\norm{\M-\M_m} = (1+\eps)\sigma_{m+1}(\M)
\end{align*}
with probability at least $1-p$.
\end{proof}

\section{The Proof of Lemma~\ref{lem:fro-shrink}} \label{appdendix:proof-balance}

\begin{proof}
The procedure of Algorithm~\ref{alg:SCOD} means $\sigma^2_i(\tX^\ui)=\sigma^2_i(\tY^\ui)=\sigma_i\big(\tX^{\ui\top}\tY^\ui\big)$. Consider that the output $\Z^\ui$ of SPM (Algorithm~\ref{alg:SPM}) is column orthonormal, then we have
\begin{align*}
  & \big\|\tX^\ui\big\|^2_F\big\|\tY^\ui\big\|_F^2  \\
= & \big\|\tX^{\ui\top}\tY^\ui\big\|_F^2 \\
= & \big\|\Z^\ui\Z^{\ui\top}\X'^{\ui\top}\Y'^\ui\big\|_F^2 \\
= & \tr\left(\Y'^{\ui\top}\X'^\ui\Z^\ui\Z^{\ui\top}\Z^\ui\Z^{\ui\top}\X'^{\ui\top}\Y'^\ui\right) \\
= & \tr\left(\Y'^{\ui\top}\X'^\ui\Z^\ui\Z^{\ui\top}\X'^{\ui\top}\Y'^\ui\right) \\
\leq & \tr\left(\Y'^{\ui\top}\X'^\ui\X'^{\ui\top}\Y'^\ui\right) \\
= & \big\|\X'^{\ui\top}\Y'^\ui\big\|_F^2 \\
\leq & \big\|\X'^\ui\big\|_F^2\big\|\Y'^\ui\big\|_F^2.
\end{align*}
\end{proof}

\section{More Details of Numerical Experiments}\label{appendix:experiment}

Our experiments are conducted on a desktop computer with Intel(R) Core(TM) i5-4570 CPU and 24GB memory. We use MATLAB 2019a to run the experiments and the operating system is Windows 10.\footnote{The code is publicly available at: \url{http://luoluo.people.ust.hk/code/SCOD.zip}} 
In the implementation of subspace power method (Algorithm~\ref{alg:SPM}), powering $\M\M^\top$ makes $(\M\M^\top)^q\M\G$ could be ill-conditioned. We include an additionally orthonormalization step after each round of multiplications to improve the stability~\cite{martinsson2010normalized,musco2015randomized}. This operation does not change the column span, so it gives an equivalent algorithm in exact arithmetic, but improves empirical performance significantly. Since $q$ is typical a small constant in practice, the additionally cost of orthonormalization is limited. 

We use cross-language datasets as we mentioned in Section~\ref{section:experiments}.
Each of dataset has alignment information of two languages at sentence-level and there are $n$ sentences in total.
We let $t$-th row of $\X$ be the bag-of-words feature of $t$-th sentence with respect to one language and $t$-th row of $\Y$ be the bag-of-words feature of the same sentence respect to the other language.

We present sketch-time comparison in Figure~\ref{figure:sketch-time}. The algorithms SFD-AMM and SCOD are much more faster than FD-AMM and COD, since FD-AMM and COD ignore the sparse structure of the input matrices. The running time of SFD-AMM and SCOD are comparable which satisfies our complexity analysis. 

\begin{figure*}[!ht]
\centering
\begin{tabular}{cccc}
     \includegraphics[scale=0.29]{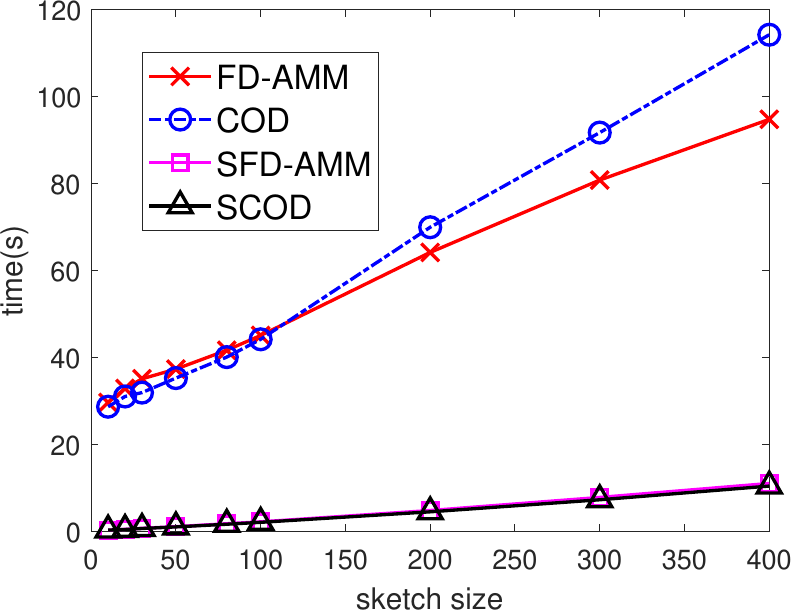} &
     \includegraphics[scale=0.29]{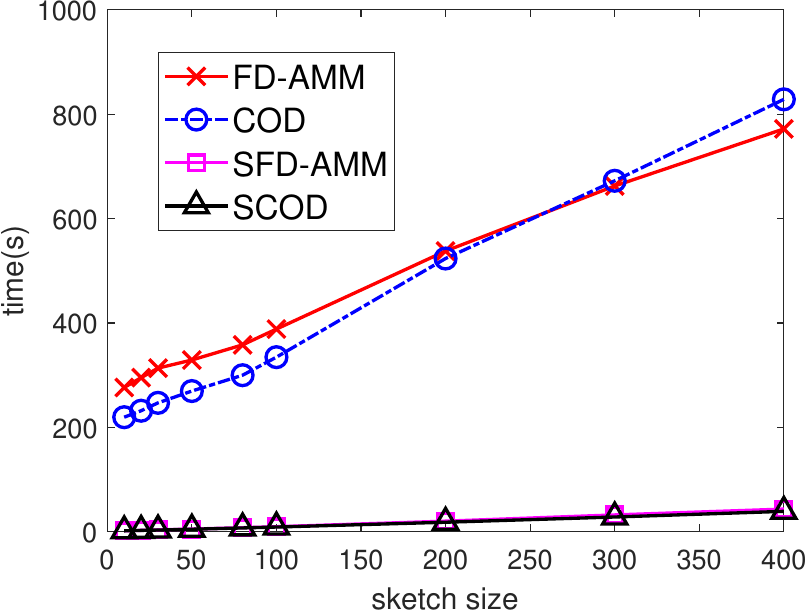} &
     \includegraphics[scale=0.29]{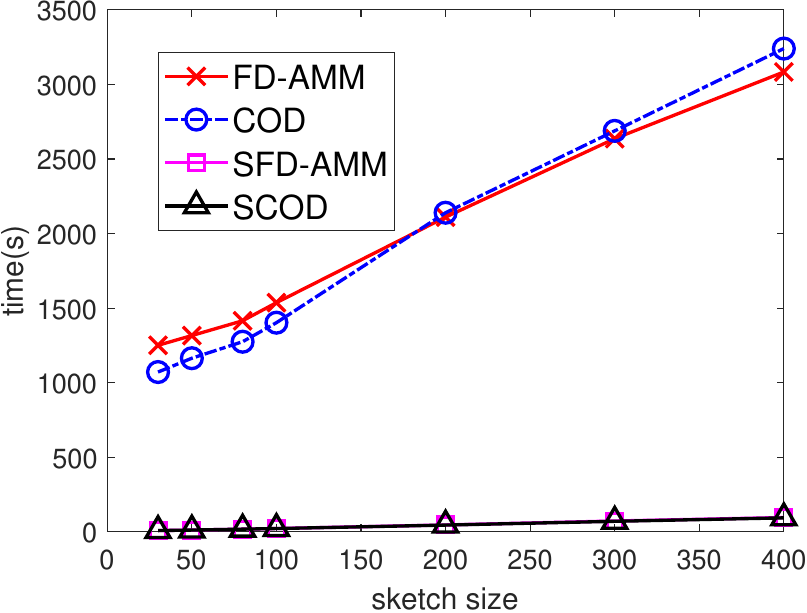} &
     \includegraphics[scale=0.29]{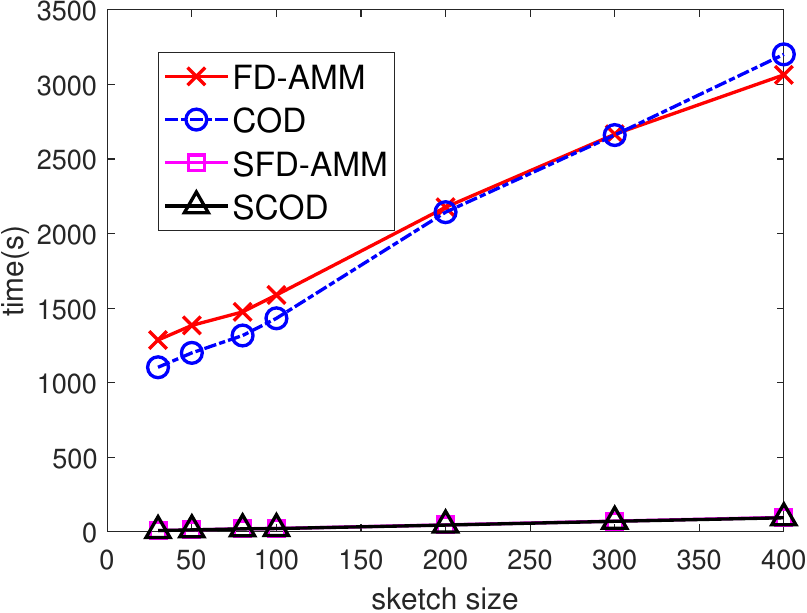} \\
     \small (a) APR (EN-FR) & 
     \small (b) PAN (EN-FR) & 
     \small (d) JRC (EN-FR) &
     \small (e) JRC (FR-ES) \\[0.1cm]
     \includegraphics[scale=0.29]{figure/jrc0-b.pdf} &
     \includegraphics[scale=0.29]{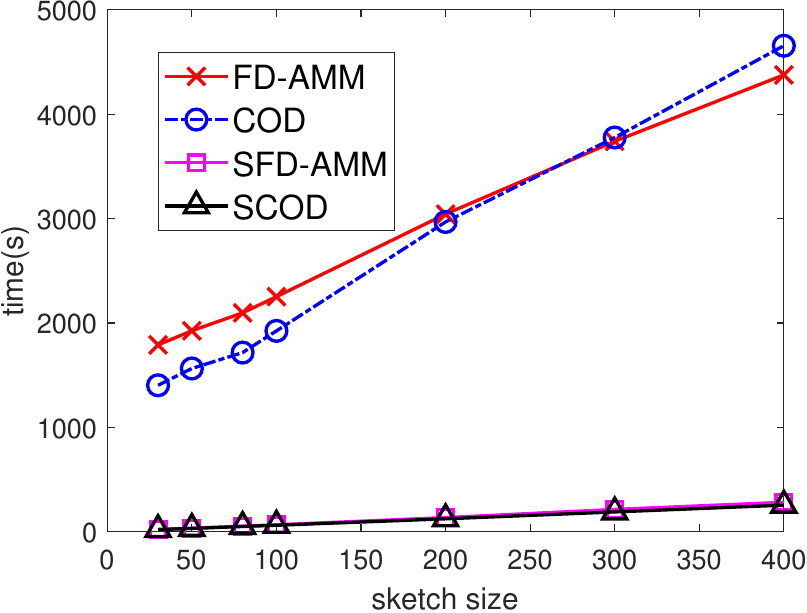} &
     \includegraphics[scale=0.29]{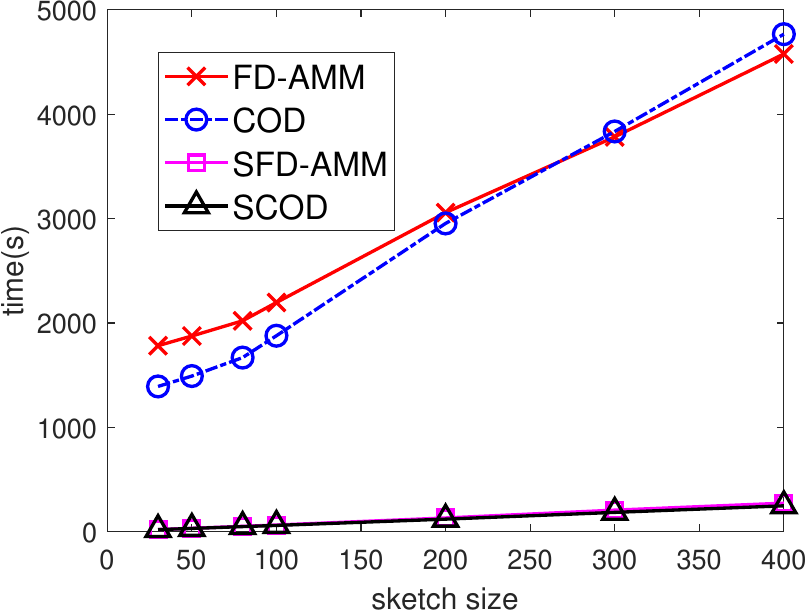} &
     \includegraphics[scale=0.29]{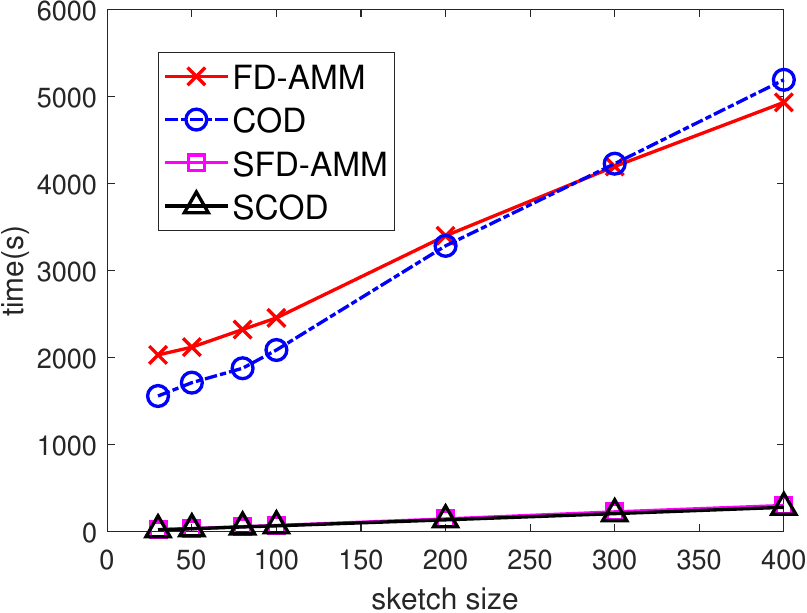} \\
     \small (f) JRC (FR-ES) & 
     \small (g) EURO (EN-FR) & 
     \small (j) EURO (EN-ES) &
     \small (i) EURO (FR-ES)
\end{tabular}
\caption{The plot of sketch size against time (s)}\label{figure:sketch-time}
\end{figure*}

\section{Additional Discussion on SFD-AMM}\label{appendix:SFD-AMM}

\citet{ghashami2016frequent} proposed a variant of FD for sketching sparse matrices called sparse frequent directions (SFD). 
Given input matrix $\Z\in\BR^{n\times d}$, the algorithm output $\C\in\BR^{m\times d}$ such that
\begin{align*}
    \norm{\Z^\top\Z-\C^\top\C} \leq \frac{1}{\alpha m-k}\left(\Norm{\Z}_F^2-\Norm{\Z_k}_F^2\right)
\end{align*}
with high probability for any $k<\alpha m$, where $\alpha$ is a constant depends on the accuracy of SPM. SFD requires $\tilde\fO\big(m\cdot\nnz(\Z) + m^2n\big)$ time complexity and $\fO(md)$ space. The procedure of SFD is similar to SCOD in the case of $\X=\Y$, but includes additional shrinking operation on the output of SPM~\cite{ghashami2016efficient} to apply the ``mergeability porperty'' of FD~\cite{ghashami2016frequent,desai2016improved} in their analysis.

For streaming AMM with sparse input, it is natural to combine the idea of SFD with FD-AMM directly which leads to the algorithm sparse FD-AMM (SFD-AMM). 
Similar to FD-AMM, SFD-AMM applies SFD on concatenated matrix $\Z=[\X, \Y]\in\BR^{n\times(d_1+d_2)}$ and its output $\C$ which can be written as $\C=[\A,\B]$, where $\A\in\BR^{n\times d_x}$ and $\B\in\BR^{n\times d_y}$.  
Then we use $\A^\top\B$ to approximate $\X^\top\Y$ that satisfies
\begin{align}\label{bound:SFD-AMM}
    \norm{\X^\top\Y-\A^\top\B} 
\leq \norm{\Z^\top\Z-\C^\top\C} 
\leq \frac{1}{\alpha m-k}\left(\Norm{\Z}_F^2-\Norm{\Z_k}_F^2\right).
\end{align}
The time complexity of SFD-AMM has the same order as SCOD since $\nnz(\Z)=\nnz(\X)+\nnz(\Y)$. It is not easy to compare the error bound (\ref{bound:SFD-AMM}) with SCOD (Theorem~\ref{thm:sparse-error}) in general because there does not exist simple relationship between the singular values of $\Z=[\X,\Y]$ and $\X^\top\Y$. However, SCOD always performs better than SFD-AMM empirically as we observed in Section~\ref{section:experiments}.

In theoretical, we can improve the time complexity of SFD-AMM to achieve the error bound of (\ref{bound:SFD-AMM}) by integrating random sampling~\cite{huang2019near}. 
However, the implementation of this strategy requires the value of $k$ is given. 
Unfortunately, it is difficult to select a suitable $k$ for streaming setting in general. In contrast, the value of $k$ in SFD-AMM or SCOD is only for theoretical analysis and it is no related to the implementation of algorithms.

\end{document}